\documentclass[english,british,envcountsame]{scrartcl}
\usepackage[T1]{fontenc}
\usepackage[latin9]{inputenc}
\usepackage{listings}
\usepackage[a4paper]{geometry}
\usepackage{xcolor}
\usepackage{pdfcolmk}
\usepackage{array}
\usepackage{varioref}
\usepackage{float}
\usepackage{multirow}
\usepackage{amsthm}
\usepackage{amsmath}
\usepackage{amssymb}
\PassOptionsToPackage{normalem}{ulem}
\usepackage{ulem}

\makeatletter

\newcommand{\noun}[1]{\textsc{#1}}
\providecommand{\tabularnewline}{\\}
\providecolor{lyxadded}{rgb}{0,0,1}
\providecolor{lyxdeleted}{rgb}{1,0,0}

\usepackage{enumitem}		
\theoremstyle{plain}
\newtheorem{thm}{\protect\theoremname}[section]
  \theoremstyle{definition}
  \newtheorem{defn}[thm]{\protect\definitionname}
  \theoremstyle{plain}
  \newtheorem{prop}[thm]{\protect\propositionname}
  \theoremstyle{plain}
  \newtheorem{lem}[thm]{\protect\lemmaname}
  \theoremstyle{definition}
  \newtheorem{example}[thm]{\protect\examplename}

\pdfoutput=1

\usepackage{stmaryrd}

\usepackage[linesnumbered,ruled,vlined,noend,nofillcomment,algo2e] {algorithm2e}
\SetArgSty{}
\usepackage{ifthen}
\usepackage{longtable}
\usepackage{array}

\usepackage{tikz}
\usetikzlibrary{arrows}
\usetikzlibrary{positioning}

\usepackage{dsfont}

\makeatother

\usepackage{babel}
  \addto\captionsbritish{\renewcommand{\definitionname}{Definition}}
  \addto\captionsbritish{\renewcommand{\examplename}{Example}}
  \addto\captionsbritish{\renewcommand{\lemmaname}{Lemma}}
  \addto\captionsbritish{\renewcommand{\propositionname}{Proposition}}
  \addto\captionsbritish{\renewcommand{\theoremname}{Theorem}}
  \addto\captionsenglish{\renewcommand{\definitionname}{Definition}}
  \addto\captionsenglish{\renewcommand{\examplename}{Example}}
  \addto\captionsenglish{\renewcommand{\lemmaname}{Lemma}}
  \addto\captionsenglish{\renewcommand{\propositionname}{Proposition}}
  \addto\captionsenglish{\renewcommand{\theoremname}{Theorem}}
  \providecommand{\definitionname}{Definition}
  \providecommand{\examplename}{Example}
  \providecommand{\lemmaname}{Lemma}
  \providecommand{\propositionname}{Proposition}
\providecommand{\theoremname}{Theorem}

\begin{document}
\titlehead{\{mvc,pb\}@fct.unl.pt\\
CENTRIA - Centre for Artificial Intelligence\\
Departamento de Informática\\
FCT/UNL Quinta da Torre 2829-516 CAPARICA - Portugal\\
Tel. (+351) 21 294 8536\\
FAX (+351) 21 294 8541}

\selectlanguage{english}%

\title{View-based propagation of decomposable constraints}

\author{Marco Correia \and Pedro Barahona}

\maketitle
\vfill{}

\thanks{The final publication is available at http://link.springer.com}

\selectlanguage{english}%
\begin{abstract}
Constraints that may be obtained by composition from simpler \foreignlanguage{british}{constraints}
are present, in some way or another, in almost every constraint program.
The decomposition of such constraints is a standard technique for
obtaining an adequate propagation algorithm from a combination of
propagators designed for simpler constraints. The decomposition approach
is appealing in several ways. Firstly because creating a specific
propagator for every constraint is clearly infeasible since the number
of constraints is infinite. Secondly, because designing a propagation
algorithm for complex constraints can be very challenging. Finally,
reusing existing propagators allows to reduce the size of code to
be developed and maintained. Traditionally, constraint solvers automatically
decompose constraints into simpler ones using additional auxiliary
variables and propagators, or expect the users to perform such decomposition
themselves, eventually leading to the same propagation model. In this
paper we explore views, an alternative way to create efficient propagators
for such constraints in a modular, simple and correct way, which avoids
the introduction of auxiliary variables and propagators.
\end{abstract}

\section{Introduction\label{sec:Motivation}}

Most specialized filtering algorithms, i.e. propagators, available
in constraint solvers target constraints with a specific algebraic
structure. When the problem to be solved has a constraint for which
no propagator exists one option is to decompose it into some logically
equivalent formula involving simpler constraints. This can be done
either by the user, or automatically if the solver supports it, by
introducing a set of auxiliary variables and propagators. In practice
the constraints which are most often decomposed are those involving
multiple distinct arithmetic or logical expressions. Throughout this
paper we will refer to constraints for which no specialized filtering
algorithm is appropriate, and which are therefore considered for decomposition,
as \emph{decomposable constraints}.

In this paper we describe view-based propagation - a method for propagating
decomposable constraints that avoids the use of auxiliary variables
and propagators. Although this method does not improve the strength
of the propagation algorithm, it results in significantly faster propagation
compared with the decompositions obtained with auxiliary variables
for the specific class of bounds completeness. An important aspect
of view based propagation is that it does not preclude the introduction
of auxiliary variables when required, meaning it can be combined with
other modeling techniques such as subexpression substitution or variable
elimination. It can also be used with specialized propagation algorithms,
e.g. used in global constraints (or even replace them), allowing the
exploitation of the decomposition approach combined with the efficiency
of such dedicated algorithms. 

The problem of propagating decomposable constraints may be approached
using knowledge compilation techniques \cite{Gent2007,Cheng2008}.
These methods create a compact tractable representation of the set
of solutions to the constraint and apply a general propagation algorithm
which filters tuples not found in this set. However, due to its expensive
runtime complexity applying these methods with arithmetic expressions
is only practical when expressions have small domains. In contrast,
view-based propagators for decomposable constraints, extending previous
proposals for propagators (e.g. \foreignlanguage{british}{indexical}
constraints \cite{Hentenryck1991,Carlson1995}), do not require exponential
memory. 

Rina Dechter \cite{Rossi2006} approaches decomposition of constraints
from a different perspective. There, the full constraint network is
taken into consideration in order to obtain generic global search
algorithms with theoretical performance guarantees. The propagation
is still time or space exponential but depends exclusively on specific
graph-based parameters of the graph describing the constraint network.
In contrast, we focus on the decomposition of a single constraint
into simpler constraints for which local propagation algorithms are
applied independently. 

View based-propagation was introduced in \cite{Correia2005} and \cite{Schulte2005}.
The former presents a general overview of a constraint solver incorporating
type polymorphism and its application for creating propagators from
decomposable expressions. The latter coins the term \emph{view} (we
originally called them \emph{polymorphic constraints}), and describes
how it can be used for creating generic propagator implementations,
that is, propagators which can be reused for different constraints.
The present work can be seen as an extension of \cite{Tack2009,Schulte2012}
for allowing the use of a particular kind of view over functions involving
multiple variables (\cite{Tack2009} restricts the use of views to
injective functions and therefore to unary functions mostly). We adapted
and extended its formalization to present and prove important properties
of our model. Additionally, our framework employs views as modeling
primitives by automatically deriving new propagators for decomposable
constraints present in a problem. In contrast, views as described
in \cite{Tack2009} are used essentially as a development tool for
increasing the number of available propagators in the library.

Compilers for constraint modeling languages \cite{Nethercote2007,Frisch05theessence,Balasubramaniam:2012:AAG:2337223.2337301}
generate efficient constraint solvers from a high level description
of a constraint problem. The solver generation process may avoid introducing
auxiliary variables for a decomposable constraint whenever a specific
propagator is available. In this paper we will show that, in the absence
of a specific propagator for a given constraint, a view-based propagator
may be an appealing alternative. 

Indexicals \cite{Hentenryck1991,Carlson1995} and constrained expressions
\cite{ILOG2003} are conceptually close to the idea of views described
in this paper. Like views, these techniques intend to support automatic
propagation of decomposable constraints, however they require extra
work from the user (in the case of indexicals) or the CPU (in the
case of constrained expressions) when compared to our approach, These
(dis)similarities will be discussed after the necessary background
is presented.

This paper is \foreignlanguage{british}{organised} as follows. The
following section overviews the main concepts of constraint programming,
introducing domains and domain approximations as well as constraint
propagators and their main properties. Section \ref{cha:Propagation-of-decomposable}
presents some examples of decomposable constraints and how to express
propagators for such constraints, discussing the properties of various
approximations, namely soundness and completeness. It also introduces
box view propagators, the propagation model that will be used throughout
the paper. Section \ref{cha:Implementation-and-Experiments} briefly
addresses two alternatives for implementing box views in strongly
typed programming languages (exploiting either subtype or parametric
polymorphism). Section \ref{sec:Experiments} presents experimental
results for a comprehensive set of benchmarks. The paper concludes
in section \ref{sec:Conclusion-and-Future} with a summary of the
results obtained and some suggestions for future work.

\selectlanguage{british}%

\section{Background: Domains and Propagators\label{cha:constraint-programming}}

This section presents the necessary concepts and notation for describing
propagation in detail.

\subsection{\label{sub:CSPs}Constraint Satisfaction Problems}

A constraint satisfaction problem (CSP) is a triple $\left\langle X,D,C\right\rangle $
where $X$ is a finite set of variables, $D$ is a finite set of variable
domains, and $C$ is a finite set of constraints. We refer to the
set of constraints involving some variable $x\in X$ as $C\left(x\right)$
and the set of variables in some constraint $c\in C$ as $X\left(c\right)$.

An n-ary constraint (on variables $x_{1}$, ... $x_{n}$) is a set
of tuples on its n variables ($n$-tuples). Although tuples of integers
will be most often used, we only expect that the elements in a tuple
are from totally ordered types. In general, we will denote an n-tuple
as $\mathbf{x}^{n}$ and a set of n-tuples as $S^{n}$ (or simply
$\mathbf{x}$ and $S$ respectively when there is no ambiguity on
its arity) and the set of tuples of constraint $c$ as $\mathrm{con}\left(c\right)$.

Implicitly, a constraint restricts the values of each of its variables.
Denoting by $\mathrm{proj}_{i}\left(S^{n}\right)$ or simply $S_{i}^{n}$
the projection of the set of tuples to their i-th element, constraint
$c$ restricts its i-th variable to take values in $\mathrm{con}\left(c\right)_{i}$.
In particular, the variable domain constraint $D\left(x\right)$ is
a special case of a unary constraint, restricting the initial values
of variable $x$.

\subsection{Domain Approximations}

Following \cite{Tack2009} we generalise the notion of domains from
single variables to more general n-ary domains and characterize their
approximations. First we denote by $\mathrm{conv}\left(S^{1}\right)$
the convex hull of some set $S^{1}$ from an ordered type $\mathcal{D}$
i.e.

\begin{eqnarray*}
\mathrm{conv}_{\mathcal{D}}\left(S^{1}\right) & = & \left\{ z\in\mathcal{D}:\min\left(S^{1}\right)\leq z\leq\max\left(S^{1}\right)\right\} 
\end{eqnarray*}

We introduce now important tuple set operations: Cartesian approximation
(see for example \cite{Ball2003}) and box approximation \cite{Benhamou1995}. 
\begin{defn}
[Cartesian approximation] The Cartesian approximation $S^{\delta}$
(or $\delta$-approximation) of a tuple set $S\subseteq\mathbb{Z}^{n}$
is the smallest Cartesian product which contains $S$, that is:
\[
S^{\delta}=\mathrm{proj}_{1}\left(S\right)\times\ldots\times\mathrm{proj}_{n}\left(S\right)
\]

\end{defn}
\begin{defn}
[Box approximation] The box approximation $S^{\beta\left(\mathcal{D}\right)}$
(or $\beta$-approximation) of a tuple set $S\subseteq\mathcal{D}^{n}$
is the smallest $n$-dimensional box containing $S$, that is:
\[
S^{\beta\left(\mathcal{D}\right)}=\mathrm{conv}_{\mathcal{D}}\left(\mathrm{proj}_{1}\left(S\right)\right)\times\ldots\times\mathrm{conv}_{\mathcal{D}}\left(\mathrm{proj}_{n}\left(S\right)\right)
\]

\end{defn}
We will address integer and real box approximations, i.e. $S^{\beta\left(\mathbb{Z}\right)}$
or $S^{\beta\left(\mathbb{R}\right)}$, and denote these box approximations
by $S^{\beta}$ and $S^{\rho}$, respectively. Additionally, we will
use the identity approximation in order to simplify notation.
\begin{defn}
[Identity approximation] The identity operator transforms a tuple
set in itself, i.e. $S^{\mathds{1}}=S$.\end{defn}
\begin{prop}
Domain approximations have the following properties, for any $n$-tuple
sets $S,S_{1},S_{2}\subseteq\mathcal{D}^{n}$ and $\Phi\in\left\{ \mathds{1},\delta,\beta,\rho\right\} $,\label{prop:approx}\end{prop}
\begin{enumerate}
\item idempotency: $\left(S^{\Phi}\right)^{\Phi}=S^{\Phi}$,
\item monotonicity: $S_{1}\subseteq S_{2}\Longrightarrow S_{1}^{\Phi}\subseteq S_{2}^{\Phi}$,
and
\item closure under intersection: $S_{1}^{\Phi}\cap S_{2}^{\Phi}=\left(S_{1}^{\Phi}\cap S_{2}^{\Phi}\right)^{\Phi}$.
Note that in general it is not the case that \foreignlanguage{english}{$S_{1}^{\Phi}\cap S_{2}^{\Phi}=\left(S_{1}\cap S_{2}\right)^{\Phi}$}.
\end{enumerate}
We may now define approximation domains, or $\Phi$-domains, as follows
\begin{defn}
[$\Phi$-domain]A tuple set $S$ is a $\Phi$-domain iff $S=S^{\Phi}$
(for $\Phi\in\left\{ \mathds{1},\delta,\beta,\rho\right\} $). 
\end{defn}
\begin{defn}
A domain $S_{1}$ is \emph{stronger} than a domain $S_{2}$ ($S_{2}$
is \emph{weaker} than $S_{1}$), iff $S_{1}\subseteq S_{2}$. $S_{1}$
is \emph{strictly stronger} than $S_{2}$ ($S_{2}$ is \emph{strictly
weaker} than $S_{1}$) iff $S_{1}\subset S_{2}$. 
\end{defn}
The following lemma (proofs for most propositions and lemmas in this
paper are given in \cite{Corr10}) shows how the previously defined
approximations are ordered for a given tuple set (or constraint). 
\begin{lem}
\label{lem:domain-approx-order-1}Let $S\subseteq\mathbb{Z}^{n}$
be an arbitrary tuple set. Then, 
\[
S=S^{\mathds{1}}\subseteq S^{\delta}\subseteq S^{\beta}\subseteq S^{\rho}
\]

\end{lem}

\subsection{Propagation}
\begin{defn}
[Propagator]\label{def:propagator} A propagator (or filter) implementing
a constraint $c\in C$ is a function $\pi_{c}:\wp\left(\mathcal{D}^{n}\right)\rightarrow\wp\left(\mathcal{D}^{n}\right)$
that is contracting (i.e. $\pi_{c}\left(S\right)\subseteq S$ for
any tuple set $S\subseteq\mathcal{D}$) and sound (it never removes
tuples from the associated constraint, i.e. $\textrm{con}\left(c\right)\cap S\subseteq\pi_{c}\left(S\right)$). 
\end{defn}
We will assume that propagators are monotonic, i.e. $\pi_{c}\left(S_{1}\right)\subseteq\pi_{c}\left(S_{2}\right)$
if $S_{1}\subseteq S_{2}$, although this property is not mandatory
in modern constraint solvers, as shown in \cite{Tack2009}.

Two other properties of propagators are of interest: idempotency and
completeness. A propagator $\pi_{c}$ is idempotent iff $\pi_{c}\left(\pi_{c}\left(S\right)\right)=\pi_{c}\left(S\right)$
for any tuple set $S$. We denote by $\pi_{c}^{\star}$ the iterated
application of propagator $\pi_{c}$ until a fixpoint is reached (hence
$\pi_{c}\left(\pi_{c}^{\star}\left(S\right)\right)=\pi_{c}^{\star}\left(S\right)$).
Additionally, if propagator $\pi_{c}$ is idempotent for any tuple
set $S$ we will refer to it as $\pi_{c}^{\star}$.

The contracting condition alone sets a very loose upper bound on the
output of a propagator. Many functions meet these requirements without
performing any useful filtering (e.g. the identity function). In general,
useful propagators aim at being complete with respect to some domain,
thus achieving some consistency on the constraint associated with
the propagator.
\begin{defn}
[Propagator completeness]\label{def:propagator-completeness} A propagator
$\pi_{c}$ for a constraint $c\in C$ is $\Phi\Psi$-complete, denoted
as $\pi_{c}^{\Phi\Psi}$, iff 
\[
\pi_{c}^{\star}\left(S\right)\subseteq\left(\textrm{con}\left(c\right)\cap S^{\Phi}\right)^{\Psi}
\]
 
\end{defn}
We should note that when considering real box approximations, $\textrm{con}\left(c\right)$
in the above definition corresponds to the relaxation of constraint
$c$ to the real numbers \cite{Tack2009}. Table \ref{tab:propagator-strength}
shows the correspondence between the constraint consistencies that
are traditionally considered \cite{Tack2009,Rossi2006,Maher2002,Benhamou1996}
and the different $\Phi\Psi$-completeness of the propagators.

\begin{table}
\hfill{}%
\begin{tabular}{c|c}
Consistency  & $\Phi\Psi$-completeness\tabularnewline
\hline 
domain  & $\delta\delta$-complete\tabularnewline
bounds($D$)  & $\delta\beta$-complete\tabularnewline
range  & $\beta\delta$-complete\tabularnewline
bounds($\mathbb{Z}$)  & $\beta\beta$-complete\tabularnewline
bounds($\mathbb{R}$)  & $\rho\rho$-complete\tabularnewline
\end{tabular}\hfill{}

\caption{\label{tab:propagator-strength}Constraint propagator completeness.}
\end{table}

\begin{example}
[Propagator completeness taxonomy]\label{exa:completeness-taxonomy}
A hierarchy of the different types of $\Phi\Psi$-completenesses described
is shown below, where propagators at the start of the arrows are stronger
than those at their end (double arrow denotes \emph{strictly stronger}
propagators for the shown tuples). The figure considers constraint
$c=\left[2x+3y=z\right]$, and the arrows are labelled with tuple
sets $S=S_{x}\times S_{y}\times S_{z}$, where the crossed-out values
are removed by the stronger but not by the weaker propagator. For
example, when applied to a domain $S=\left\{ \left\langle 0,1\right\rangle ,\left\langle 0\right\rangle ,\left\langle 0,1,2\right\rangle \right\} $,
a domain ($\delta\delta$-complete) propagator is able to prune tuples
$\left\langle 0,0,1\right\rangle $ and $\left\langle 1,0,1\right\rangle $
that a bounds($D$) ($\delta\beta$-complete) propagator cannot.
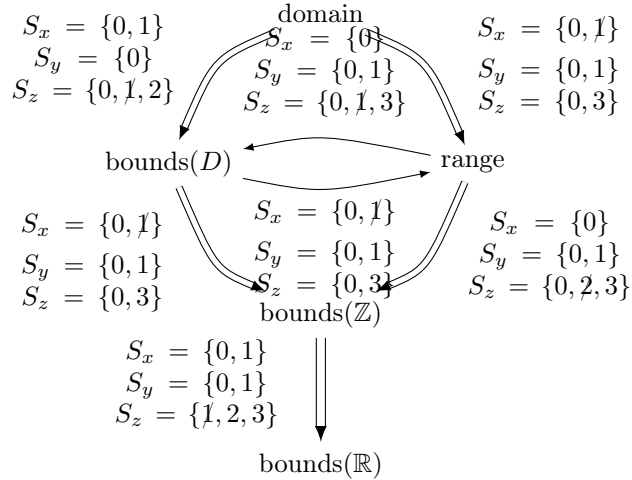
\begin{figure}[H]
\hfill{}\begin{tikzpicture}[scale=2,font=\small] 
\path (0,1) node (dd) {domain} -- 
             (-1,0) node (db) {bounds($D$)} -- 
             (0,-1) node (bb) {bounds($\mathbb{Z}$)} --
             (0,-2) node (r) {bounds($\mathbb{R}$)};
\path (0,1) -- (1,0) node  (bd) {range} -- (0,-1);

\path [draw,double,double distance=2pt,-latex] (dd) .. controls (-0.7,0.7) .. 
      node [midway,anchor=east,text width=3cm,text centered]
           {$S_x=\{ 0,1 \}$ \\
            $S_y=\{ 0 \}$ \\
            $S_z=\{ 0,\xout{1},2 \}$} (db);
\draw [draw,double,double distance=2pt,-latex] (dd) .. controls (0.7,0.7) .. 
      node [midway,anchor=west,text width=3cm,text centered] 
           {$S_x=\{ 0,\xout{1} \}$ \\ 
            $S_y=\{ 0,1 \}$ \\
            $S_z=\{ 0,3 \}$} (bd);
\path [draw,double,double distance=2pt,-latex] (db) .. controls (-0.7,-0.7) .. 
      node [midway,anchor=east,text width=3cm,text centered] 
           {$S_x=\{ 0,\xout{1} \}$ \\
            $S_y=\{ 0,1 \}$ \\
            $S_z=\{ 0,3 \}$} (bb);
\path [draw,double,double distance=2pt,-latex] (bd) .. controls (0.7,-0.7) .. 
      node [midway,anchor=west,text width=3cm,text centered] 
           {$S_x=\{ 0 \}$ \\
            $S_y=\{ 0,1 \}$ \\
            $S_z=\{ 0,\xout{2},3 \}$} (bb);
\path [draw,double,double distance=3pt,-latex] (bb) -- 
        (r) node [midway,anchor=east,text width=3cm,text centered] 
   {$S_x=\{ 0,1 \}$ \\
    $S_y=\{ 0,1 \}$ \\
    $S_z=\{ \xout{1},2,3 \}$};
\path [draw,-latex] (bd) .. controls (0,0.2) .. 
      node [midway,anchor=south,text width=3cm,text centered] 
           {$S_x=\{ 0 \}$ \\
            $S_y=\{ 0,1 \}$ \\
            $S_z=\{ 0,\xout{1},3 \}$} (db);
\path [draw,-latex] (db) .. controls (0,-0.2) .. 
      node [midway,anchor=north,text width=3cm,text centered] 
           {$S_x=\{ 0,\xout{1} \}$ \\
            $S_y=\{ 0,1 \}$ \\
            $S_z=\{ 0,3 \}$} (bd);

\end{tikzpicture}\hfill{}
\end{figure}

\end{example}
Notice that propagators achieving bounds($D$) and range consistency
are incomparable, as illustrated by the tuple sets that label the
two \emph{single} arrows with opposing directions. Note also that
bounds$\left(\mathbb{R}\right)$ completeness is weaker than bounds$\left(\mathbb{Z}\right)$
completeness, as shown in the figure. In particular, whereas a bounds$\left(\mathbb{Z}\right)$
complete propagator prunes value $z=1$ from the tuples in the figure,
this is not the case with bounds($\mathbb{R}$) complete propagators
since $\left\langle 0.5,0,1\right\rangle $ is a solution of the relaxation
of constraint to the real numbers.\selectlanguage{british}%

\section{\label{cha:Propagation-of-decomposable}Views for Propagation of
Decomposable Constraints}

This section describes propagation of decomposable constraints as
a function of \emph{views}, an abstraction representing the pointwise
evaluation of a function over a given (tuple) set. First we characterize
decomposable constraints (\ref{sec:Decomposable-constraints}). Then
we introduce views (\ref{sec:Views}), and show that they may be used
to express sound propagators for decomposable constraints (\ref{sec:View-based_Propagators}).
Section \ref{sec:Properties_of_Propagators} briefly discusses propagation
strength of view-based propagators, a subject that will be revisited
later in section \ref{sub:Completeness-revisited} after the implementation
details are presented. Finally, we introduce box view-based propagators
(\ref{sec:Box_view_Propagators}), a special case of view-based propagators
that may be implemented efficiently in a straightforward way, as described
in the rest of the paper.

\subsection{Decomposable constraints\label{sec:Decomposable-constraints}}

As previously discussed, constraint solvers do not have a specific
propagator for all possible constraints. Constraints that cannot be
captured by a single specific propagator are usually decomposed into
a logically equivalent conjunction of simpler constraints for which
specific propagators exist. Many types of constraints used in practice
are decomposable in this sense. Below are some examples.
\begin{example}
[Arithmetic constraints] These are probably the most common type
of decomposable constraints. Examples are constraints involving a
sum of variables, e.g. $\left[\sum_{i}x_{i}=k\right]$, a linear combination,
e.g. $\left[\sum_{i}a_{i}x_{i}\geq k\right]$, or a product, e.g.
$\left[\prod_{i}x_{i}=k\right]$. While some solvers do have specific
propagators for these constraints (e.g. by using the notion of linear
relation), that approach does not work for more irregular arithmetic
constraints that may include an arbitrary combination of arithmetic
operators, such as for example $\left[\left|x_{1}-x_{2}\right|=2x_{3}\right]$. 
\end{example}
\begin{example}
[Boolean constraints] Boolean constraints involve Boolean domain
variables or expressions, for example a disjunction of variables $\left[\bigvee_{i}x_{i}\right]$,
or more complex expressions on Boolean constraints such as disjunctions
$\left[\bigvee_{i}\left(x_{i}>0\right)\right]$, logical implications
$\left[x>0\Rightarrow y<0\right]$ or equivalences $\left[x>0\Leftrightarrow y\right]$. 
\end{example}
\begin{example}
[Counting constraints] Counting constraints restrict the number of
occurrences of some values within a collection of variables, for example
the \noun{exactly}$\left(\mathbf{x},v,c\right)$ constraint $\left[\sum_{i}\left(x_{i}=v\right)=c\right]$,
or the \noun{among}$\left(\mathbf{x},V,c\right)$ constraint $\left[\sum_{i}\left(x_{i}\in V\right)=c\right]$. 
\end{example}
\begin{example}
[Data constraints]\label{exa:data-constraints}Also known as \emph{ad-hoc}
constraints, they represent an access to an element of a data structure
(a table, a matrix, a relation) \cite{Beldiceanu2010}. The most common
constraint in this class is the \noun{element}$\left(\mathbf{x},i,y\right)$
constraint enforcing $\left[x_{i}=y\right]$ where $i$ is a variable.
Decomposable constraints involving this constraint are for example
$\left[x_{i}\geq y\right]$ or $\left[x_{i}-x_{j}\geq0\right]$. 
\end{example}
We propose to express decomposable constraints as a composition of
functions. For this purpose we will make extensive use of functions
that evaluate to tuples, i.e. $f:\mathbb{Z}^{n}\rightarrow\mathbb{Z}^{k}$
where $k\geq1$, together with tuple projections, and the usual composition
and Cartesian product of functions:
\begin{defn}
 [Functional composition]Functional composition is denoted by operator
$\circ$ as usual: $\left(f\circ g\right)\left(x\right)=f\left(g\left(x\right)\right)$. 
\end{defn}
\begin{defn}
[Cartesian product of functions]\label{def:cartesian-functions}Cartesian
product of functions is denoted by operator $\times$ as follows:
$\left(f\times g\right)\left(x\right)=\left\langle f\left(x\right),g\left(x\right)\right\rangle $.
If $x$ is a tuple we may write 
\[
\left(f\times g\right)\left(x_{1},\ldots,x_{n}\right)=\left\langle f\left(x_{1},\ldots,x_{n}\right),g\left(x_{1},\ldots,x_{n}\right)\right\rangle 
\]

\end{defn}
We exemplify our approach to decomposition for some of the constraints
given above. In the following examples let $x$, $x_{i}$ and $y$
represent variables, $\mathbf{x}$ represent a tuple of variables
$x_{i}$, $a$ a constant, and $p_{i}\left(\mathbf{x}\right)=x_{i}$
the projection operator. 
\begin{example}
[Equality constraint]Let constraint $c_{\textrm{eq}}\left(\mathbf{x}\right)=\left[x_{1}=x_{2}\right]$
be the binary constraint stating that variables $x_{1}$ and $x_{2}$
must take the same value. A unary equality constraint $c_{\textrm{eqc}}\left(x,a\right)=\left[x=a\right]$
may be obtained by $c_{\textrm{eqc}}=\left[c_{\textrm{eq}}\circ\left(p_{1}\times f_{a}\right)\right]$,
where $f_{a}\left(\mathbf{x}\right)=a$. 
\end{example}
\begin{example}
[Sum constraint] \label{exa:sum-constraint}Let $f\left(\mathbf{x}\right)=x_{1}+x_{2}$.
A sum of three variables is represented by $f\circ\left(f\times p_{3}\right)$.
The generalization to a sum of $n$ variables is defined as 
\[
g=f\circ\left(f\times p_{3}\right)\circ\left(f\times p_{3}\times p_{4}\right)\circ\ldots\circ\left(f\times p_{3}\times p_{4}\times\ldots\times p_{n}\right)
\]
 Therefore we may decompose a constraint for a sum of $n$ variables
$c_{\textrm{sum0}}\left(\mathbf{x}^{n}\right)=\left[\sum_{i}^{n}x_{i}=0\right]$
as $c_{\textrm{sum0}}=\left[c_{\textrm{eq0}}\circ g\right]$, where
$c_{\textrm{eq0}}\left(\mathbf{x}\right)=c_{\textrm{eqc}}\left(\mathbf{x},0\right)$
defined in the previous example. 
\end{example}
\begin{example}
[Linear constraint]\label{exa:linear-constraint}A linear constraint
$c_{\textrm{lin0}}\left(\mathbf{a}^{n},\mathbf{x}^{n}\right)=\left[\sum_{i=1}^{n}a_{i}x_{i}=0\right]$
may be decomposed as follows. Let $f_{1},\ldots,f_{n}:\mathbb{Z}^{n}\rightarrow\mathbb{Z}$,
where $f_{i}\left(\mathbf{x}^{n}\right)=a_{i}x_{i}$. Then, $c_{\textrm{lin0}}=\left[c_{\textrm{sum0}}\circ\left(f_{1}\times\ldots\times f_{n}\right)\right]$,
where $c_{\textrm{sum0}}$ is defined in the previous example. 
\end{example}
\begin{example}
[Arbitrary arithmetic constraint]Let $f\left(\mathbf{x}\right)=x_{1}-x_{2}$,
$g\left(x\right)=\left|x\right|$, $h\left(x\right)=2x$. The arithmetic
constraint $c\left(\mathbf{x}^{3}\right)=\left[\left|x_{1}-x_{2}\right|=2x_{3}\right]$
may be represented as 
\[
c=\left[c_{\textrm{eq}}\circ\left(\left(g\circ f\circ\left(p_{1}\times p_{2}\right)\right)\times\left(h\circ p_{3}\right)\right)\right]
\]

\end{example}
\begin{example}
[Exactly constraint]This constraint, represented by $\left[\sum_{i}\left(x_{i}=v\right)=c\right]$,
may be obtained similarly to example \ref{exa:linear-constraint}
but where $f_{i}\left(\mathbf{x}^{n}\right)=\left[x_{i}=v\right]$
and the constraint $c_{\textrm{eqc}}$ is used instead of $c_{\textrm{eq0}}$. 
\end{example}

\subsection{Views\label{sec:Views}}

In most constraint solvers, domains of subexpressions occurring in
constraints are not directly available to propagators, which are designed
to work exclusively with variable domains. In these solvers, an \foreignlanguage{british}{offline}
modeling phase is responsible for obtaining an equivalent CSP where
all constraints are \emph{flattened}, that is where each subexpression
appearing in a constraint is replaced by an auxiliary variable whose
domain is the domain of the subexpression.

Alternatively, we present a conceptual model that considers explicit
subexpression domains while abstracting on how they are computed and
maintained. This will allow us to perform a theoretical analysis of
the propagation on the constraint decomposition, regardless of the
method used for representing the subexpression domains. 

We begin by defining a view as an abstraction which captures the domain
of a subexpression in a constraint.
\begin{defn}
[View]\label{def:view}A view over a function $f:\mathbb{Z}^{n}\rightarrow\mathbb{Z}^{k}$
is a pair $\varphi=\left\langle \varphi_{f}^{+},\varphi_{f}^{-}\right\rangle $
of two functions, the \emph{image function} $\varphi_{f}^{+}:\wp\left(\mathbb{Z}^{n}\right)\rightarrow\wp\left(\mathbb{Z}^{k}\right)$,
and the \emph{object function} $\varphi_{f}^{-}:\wp\left(\mathbb{Z}^{k}\right)\rightarrow\wp\left(\mathbb{Z}^{n}\right)$,
defined as 
\begin{eqnarray*}
\varphi_{f}^{+}\left(S^{n}\right) & = & \left\{ f\left(\mathbf{x}^{n}\right)\in\mathbb{Z}^{k}:\mathbf{x}^{n}\in S^{n}\right\} ,\forall S\subseteq\mathbb{Z}^{n}\\
\varphi_{f}^{-}\left(S^{k}\right) & = & \left\{ \mathbf{x}^{n}\in\mathbb{Z}^{n}:f\left(\mathbf{x}^{n}\right)\in S^{k}\right\} ,\forall S\subseteq\mathbb{Z}^{k}
\end{eqnarray*}

\end{defn}
A view $\varphi_{f}$ is therefore defined by considering the pointwise
application of $f$ over a given set. The image function computes
a set of images of $f$, that is the set resulting from applying $f$
to all the points in the given set. The object function does the inverse
transformation, computing the set of objects of $f$ for which its
image is in the given set. 
\begin{example}
\label{exa:Applying-a-phi-view}Consider the view over function $f\left(x\right)=x+1$
and the unary tuple set $S=\left\{ 1,2,3\right\} $: 
\begin{eqnarray*}
\varphi_{f}^{+}\left(\left\{ 1,2,3\right\} \right) & = & \left\{ 2,3,4\right\} \\
\varphi_{f}^{-}\left(\left\{ 2,3,4\right\} \right) & = & \left\{ 1,2,3\right\} 
\end{eqnarray*}

\end{example}
Notice some similarities of views with indexicals introduced in \cite{Hentenryck1991,Carlson1995}
to create propagators for arithmetic constraints. An indexical roughly
corresponds to the image function $\varphi^{+}$ in the sense that
it computes the domain of an expression. However, unlike views, indexicals
do not define the inverse transformation $\varphi^{-}$ and therefore
are less powerful - representing a decomposable constraint using indexicals
requires the additional definition of the projection of the object
function for each variable in the constraint, even if this projection
may be performed automatically as shown in \cite{Carlson1995}.

As discussed earlier, propagating a constraint may require consulting
and updating the domain of a subexpression. A view over the subexpression
defines these operations:
\begin{example}
\label{exa:Applying-a-phi-view-1}Consider a view $\varphi_{g}$ over
function $g\left(x,y\right)=x+y$, and constraint $c=\left[x_{1}+x_{2}=x_{3}\right]$
on variables $x_{1}$, $x_{2}$, $x_{3}$ with domains $D\left(x_{1}\right)=D\left(x_{2}\right)=D\left(x_{3}\right)=\left\{ 1,2,3\right\} $.
Function $\varphi_{g}^{+}$ may be used to compute the subexpression
domain $D\left(x_{1}+x_{2}\right)$ from the variable domains $D\left(x_{1}\right)$
and $D\left(x_{2}\right)$, while function $\varphi_{g}^{-}$ may
be used to specify the set of values for the variables $x_{1}$, $x_{2}$
that satisfy constraint $c$. 
\begin{eqnarray*}
\varphi_{g}^{+}\left(D\left(x_{1}\right)\times D\left(x_{2}\right)\right) & = & \left[2\ldots6\right]\\
\varphi_{g}^{-}\left(\left[2\ldots6\right]\right) & = & \left\{ \left\langle x,y\right\rangle \in\mathbb{Z}^{2}:x+y\in\left[2\ldots6\right]\right\} 
\end{eqnarray*}

\end{example}
Propagation consists of removing inconsistent values from the domain
of the variables. In this sense, the computation of $\varphi_{f}^{-}$,
i.e. the set of consistent assignments, must be intersected with the
original domain to guarantee that the propagation is \emph{contracting}.
The following definition captures exactly that.
\begin{defn}
[Contracting object function]\label{def:phihat-view}Let $S_{1}\in\mathbb{Z}^{n}$,
$S_{2}\in\mathbb{Z}^{k}$, be two arbitrary tuple sets and $f:\mathbb{Z}^{n}\rightarrow\mathbb{Z}^{k}$
an arbitrary function. Then, 
\[
\widehat{\varphi}_{f}\left(S_{2},S_{1}\right)=\varphi_{f}^{-}\left(S_{2}\right)\cap S_{1}
\]
\end{defn}
\begin{example}
The result of evaluating and updating the domain $D\left(x_{1}+x_{2}\right)$,
as defined in the previous example, may be formalized using views
as follows: 
\begin{eqnarray*}
\varphi_{g}^{+}\left(D\left(x_{1}\right)\times D\left(x_{2}\right)\right) & = & \left[2\ldots6\right]\\
\widehat{\varphi}_{g}\left(\left[2\ldots6\right],D\left(x_{1}\right)\times D\left(x_{2}\right)\right) & = & D\left(x_{1}\right)\times D\left(x_{2}\right)
\end{eqnarray*}

\end{example}

\subsection{View-based Propagators\label{sec:View-based_Propagators}}

Propagators for decomposable constraints may be obtained by exploring
the concept of views introduced in the previous section. As illustrated
above, an $n$-ary decomposable constraint may be regarded as a special
case of functions of the form $c\circ f$ where $f$ is a tuple-function
$f:\mathbb{Z}^{n}\rightarrow\mathbb{Z}^{k}$ and $c$ a constraint,
mapping k-tuples to the Boolean domain. 
\begin{defn}
\label{defn:prop-decomp-idemp} Let $c$ be a $k$-ary constraint
for which $\pi_{c}$ is a propagator and $f$ be a tuple-function
$f:\mathbb{Z}^{n}\rightarrow\mathbb{Z}^{k}$. A view-based propagator
for constraint $c\circ f$, denoted as $\check{\pi}_{c\circ f}$,
is obtained by 
\[
\check{\pi}_{c\circ f}\left(S^{n}\right)=\widehat{\varphi}_{f}\left(\pi_{c}^{\mathds{1}\mathds{1}\star}\left(\varphi_{f}^{+}\left(S^{n}\right)\right),S^{n}\right)
\]
\end{defn}
\begin{example}
\label{exa:view-propagator-composed}Consider the view-based propagator
for the constraint $d=\left[x_{1}+x_{2}=x_{3}\right]$, and let us
analyse the propagation achieved on the tuple set $S_{1}=\left\{ \left\langle 1,2,3\right\rangle ,\left\langle 4,5,6\right\rangle \right\} $.

We note that the constraint $d$ can be decomposed into a simpler
constraint $c=\left[x_{1}=x_{2}\right]$ and function $g=f\times p_{3}$,
where function g is the Cartesian product of function $f\left(\mathbf{x}\right)=x_{1}+x_{2}$
applied to the 2 initial elements of the tuples $\left\langle x_{1},x_{2},x_{3}\right\rangle $
and $p_{3}$ the projection to their third element.

Function $\varphi_{g}^{+}$ applies the addition and projection operations
to the original set, $S_{2}=\varphi_{g}^{+}\left(S_{1}\right)=\left\{ \left\langle 3,3\right\rangle ,\left\langle 9,6\right\rangle \right\} $.
The resulting set $S_{2}$ is filtered by propagator $S_{3}=\pi_{c}^{\mathds{1}\mathds{1}\star}\left(S_{2}\right)=\left\{ \left\langle 3,3\right\rangle \right\} $,
and transformed back into $\widehat{\varphi}_{g}\left(S_{3},S_{1}\right)=\left\{ \left\langle 1,2,3\right\rangle \right\} $.
\end{example}

\subsection{Approximate view-based propagators\label{sec:Properties_of_Propagators}}

In this section we present approximate view-based propagators and
relate them to the completeness classes introduced in table \vref{tab:propagator-strength}.
\begin{defn}
[$\Phi\Psi-complete$ view-based propagator]A $\Phi\Psi$ view-based
propagator for a constraint $c\circ f$ is defined as 
\begin{eqnarray*}
\check{\pi}_{c\circ f}^{\Phi\Psi}\left(S\right) & = & \widehat{\varphi}_{f}\left(\pi_{c}^{\mathds{1}\mathds{1}\star}\circ\varphi_{f}^{+}\left(S^{\Phi}\right),S^{\Phi}\right)^{\Psi}\cap S\textrm{, where \ensuremath{\Phi},\ensuremath{\Psi\in\left\{  \mathds{1},\delta,\beta\right\} } }
\end{eqnarray*}
 \label{def:phi-psi-view-based-prop}
\end{defn}
Intuitively, $\Phi\Psi$-completeness of a view-based propagator for
a decomposable constraint of the form $c\circ f_{1}\dots\circ f_{m}$
is obtained by approximating the input of the image function $\varphi_{f}^{+}$
and the output of the object function $\widehat{\varphi}_{f}$, and
not approximating the remaining view functions or propagators involved.
For these view-based propagators the following property can be proved
\cite{Corr10}.
\begin{prop}
\label{pro:view-propagator-composed-idemp}A $\Phi\Psi$ view-based
propagator for a constraint $c\circ f$ is a $\Phi\Psi$-complete
and idempotent propagator for $c\circ f$. 
\end{prop}
The use of the above proposition is rather limited since it applies
only to a propagator $\pi_{c}$ that is $\mathds{1}\mathds{1}$-complete
and idempotent. Achieving such completeness is usually intractable
in time and/or space in general, inequality constraints being a notable
exception. Moreover other approximations are performed in practical
view-based propagators, which are now presented.

\subsection{Box view-based propagators\label{sec:Box_view_Propagators}}

A box view-based propagator (or simply box view propagator) is a relaxation
of a $\beta\beta$-complete view-based propagator. \foreignlanguage{british}{In
addition} to the $\beta$-approximations already presented, it
\begin{enumerate}
\item $\beta$-approximates the output of the image function $\varphi^{+}$, 
\item $\beta$-approximates the input of the contracting object function
$\widehat{\varphi}$, 
\item uses a $\beta\beta$-complete and idempotent propagator for $c$
\end{enumerate}
To simplify its formalization, we use the following notation for special
applications of $\Phi$-approximations to views: $ $
\begin{eqnarray*}
\Phi_{f}^{+}\left(S\right) & = & \left(\varphi_{f}^{+}\left(S^{\Phi}\right)\right)^{\Phi}\\
\widehat{\Phi}_{f}\left(S_{1},S_{2}\right) & = & \left(\widehat{\varphi}_{f}\left(S_{1}^{\Phi},S_{2}^{\Phi}\right)\right)^{\Phi}
\end{eqnarray*}

We may now formalize box view-based propagators.
\begin{defn}
[Box view-based propagator]\label{box-view-model}A box view-based
propagator for a constraint $c\circ f$ is defined as
\[
\pi_{c\circ f}^{\square}\left(S\right)=\widehat{\beta}_{f}\left(\pi_{c}^{\beta\beta\star}\left(\beta_{f}^{+}\left(S\right)\right),S\right)\cap S
\]
Box view-based propagators allow very efficient implementations since
$\beta$-domains may be stored in constant space (the domain is fully
characterized by its lower and upper bound), and computing views on
$\beta$-domains does not depend on the size of the domain for most
functions. \end{defn}
\begin{example}
Let $f\left(x\right)=\left|x\right|$ and $S_{1},S_{2}\subseteq\mathbb{Z}$
two arbitrary sets of unary tuples, and assume $\left\lfloor S_{2}\right\rfloor \geq0$.
By definition, 
\begin{eqnarray*}
\beta_{f}^{+}\left(S_{1}\right) & = & \left\{ \left|x\right|:x\in S_{1}^{\beta}\right\} ^{\beta}\\
\widehat{\beta}_{f}\left(S_{2},S_{1}\right) & = & \left\{ x:\left|x\right|\in S_{2}^{\beta}\wedge x\in S_{1}^{\beta}\right\} ^{\beta}
\end{eqnarray*}
 This definition may suggest that evaluating these functions would
take linear time, but in fact they may be computed in constant time
assuming that finding the minimum and maximum of $S_{1}$, $S_{2}$
takes constant time: 
\begin{eqnarray*}
\beta_{f}^{+}\left(S_{1}\right) & = & \begin{cases}
\left[\left\lfloor S_{1}\right\rfloor \ldots\left\lceil S_{1}\right\rceil \right] & \Leftarrow\left\lfloor S_{1}\right\rfloor >0\\
\left[-\left\lceil S_{1}\right\rceil \ldots-\left\lfloor S_{1}\right\rfloor \right] & \Leftarrow\left\lceil S_{1}\right\rceil <0\\
\left[0\ldots\max\left(-\left\lfloor S_{1}\right\rfloor ,\left\lceil S_{1}\right\rceil \right)\right] & \textrm{otherwise}
\end{cases}\\
\widehat{\beta}_{f}\left(S_{2},S_{1}\right) & = & \begin{cases}
\left[\max\left(\left\lfloor S_{2}\right\rfloor ,\left\lfloor S_{1}\right\rfloor \right)\ldots\min\left(\left\lceil S_{2}\right\rceil ,\left\lceil S_{1}\right\rceil \right)\right] & \Leftarrow\left\lfloor S_{1}\right\rfloor >0\\
\left[\max\left(-\left\lceil S_{2}\right\rceil ,\left\lfloor S_{1}\right\rfloor \right)\ldots\min\left(-\left\lfloor S_{2}\right\rfloor ,\left\lceil S_{1}\right\rceil \right)\right] & \Leftarrow\left\lceil S_{1}\right\rceil <0\\
\left[\max\left(-\left\lceil S_{2}\right\rceil ,\left\lfloor S_{1}\right\rfloor \right)\ldots\min\left(\left\lceil S_{2}\right\rceil ,\left\lceil S_{1}\right\rceil \right)\right] & \textrm{otherwise}
\end{cases}
\end{eqnarray*}

\end{example}
A thorough analysis of the propagation achieved with box view propagators
is complex and dependent on the constraints involved (see \cite{Corr10}).
Here we only present a sufficient condition for a box view based propagator
to be bounds$\left(\mathbb{R}\right)$ complete.
\begin{prop}
A box view-based propagator for $c\circ f$ is always at least bounds$\left(\mathbb{R}\right)$
complete if $f$ is continuous. \end{prop}
\begin{proof}
\begin{flalign*}
\pi_{c\circ f}^{\square}\left(S\right)= & \widehat{\beta}_{f}\left(\pi_{c}^{\beta\beta\star}\left(\beta_{f}^{+}\left(S\right)\right),S\right)\cap S & \textrm{(def. \ref{box-view-model}})\\
= & \widehat{\beta}_{f}\left(\left(\textrm{con}\left(c\right)\cap\left(\beta_{f}^{+}\left(S\right)\right)^{\beta}\right)^{\beta},S\right)\cap S & \textrm{(def.\ref{def:propagator-completeness}})
\end{flalign*}
Given that $\left(\beta_{f}^{+}\left(S\right)\right)^{\beta}=\beta_{f}^{+}\left(S\right)$
and $\beta_{f}^{-}\left(S^{\beta}\right)=\beta_{f}^{-}\left(S\right)$,
\begin{flalign*}
\pi_{c\circ f}^{\square}\left(S\right)= & \widehat{\beta}_{f}\left(\textrm{con}\left(c\right)\cap\beta_{f}^{+}\left(S\right),S\right)\cap S
\end{flalign*}
Since $\textrm{con}\left(c\right)\subseteq\textrm{con}_{\mathbb{R}}\left(c\right)$,
$S^{\beta}\subseteq S^{\rho}$ for any $S\subseteq\mathbb{Z}^{n}$,
and all operators are monotonic,
\begin{flalign*}
 & \pi_{c\circ f}^{\square}\left(S\right)\subseteq\widehat{\rho}_{f}\left(\textrm{con}_{\mathbb{R}}\left(c\right)\cap\rho_{f}^{+}\left(S\right),S\right)\cap S
\end{flalign*}
Note that when $f$ is continuous the following is true,
\begin{flalign*}
 & \rho_{f}^{+}\left(S\right)=\left(\varphi_{f}^{+}\left(S^{\rho}\right)\right)^{\rho}=\varphi_{f}^{+}\left(S^{\rho}\right)\\
 & \widehat{\rho}_{f}\left(S_{1},S_{2}\right)=\left(\widehat{\varphi}_{f}\left(S_{1}^{\rho},S_{2}\right)\right)^{\rho}=\left(\widehat{\varphi}_{f}\left(S_{1},S_{2}\right)\right)^{\rho}
\end{flalign*}
Rewriting the above expression using these equivalences gives,
\begin{flalign*}
\pi_{c\circ f}^{\square}\left(S\right)\subseteq & \left(\widehat{\varphi}_{f}\left(\textrm{con}_{\mathbb{R}}\left(c\right)\cap\varphi_{f}^{+}\left(S^{\rho}\right),S\right)\right)^{\rho}\cap S\\
\subseteq & \pi_{c\circ f}^{\rho\rho\star}\left(S\right) & \textrm{(def.\ref{def:propagator-completeness}})
\end{flalign*}

\end{proof}
\selectlanguage{british}%
So far we have been discussing view based propagators obtained from
composition of views with propagators that are idempotent. Some practical
propagators are non-idempotent (or have non-idempotent implementations)
for efficiency reasons\foreignlanguage{english}{. Unfortunately, the
strength of a view based propagator for $c\circ f$ is influenced
by the idempotency of the underlying propagator $\pi_{c}$. }
\begin{prop}
\label{prop:prop-strong-idemp-1}A box view-based propagator for a
constraint $c\circ f$ is stronger than a corresponding view-based
propagator using a non-idempotent propagator $\pi_{c}$ for $c$.
Moreover, it may be strictly stronger.
\end{prop}
The fact that a view-based propagator for a constraint $c\circ f$
with an idempotent propagator for $c$ achieves at least the same
pruning as with a non-idempotent propagator for $c$ is not surprising
since $\pi_{c}^{\star}$ is always stronger than $\pi_{c}$ and all
other involved functions and operators are monotonic. The following
example shows it can indeed achieve more pruning.
\begin{example}
Let $c\circ f=\left[2\cdot x_{1}\cdot x_{2}=x_{3}\right]$ be a decomposable
constraint, where $c=\left[2\cdot x_{1}=x_{2}\right]$ and $f\left(x_{1},x_{2},x_{3}\right)=\left\langle x_{1}\cdot x_{2},x_{3}\right\rangle $.
A box view-based propagator $\pi_{c\circ f}^{\square}$ applied to
$D=\left[2\ldots3\right]\times\left[2\ldots3\right]\times\left[9\ldots15\right]$
leaves the domains of $x_{1}$ and $x_{2}$ unchanged, but prunes
the domain of $x_{3}$ since 
\begin{eqnarray*}
\beta_{f}^{+}\left(D\right) & = & \left[4\dots9\right]\times\left[9\dots15\right]\\
\pi_{c}^{\beta\beta\star}\left(\left[4\dots9\right]\times\left[9\dots15\right]\right) & = & \left[5\dots7\right]\times\left[10\dots14\right]\\
\widehat{\beta}_{f}\left(\left[5\dots7\right]\times\left[10\dots14\right],D\right)\cap D & = & \left[2\dots3\right]\times\left[2\dots3\right]\times\left[10\dots14\right]
\end{eqnarray*}

Now consider the following non-idempotent propagator $\pi_{c}^{\beta\beta}$
for the same constraint $c$: 
\begin{eqnarray*}
\pi_{c}^{\beta\beta} & = & \begin{cases}
D\left(x_{3}\right) & \leftarrow D\left(x_{3}\right)\cap\left[\left\lfloor D\left(x_{1}\times x_{2}\right)\right\rfloor \times2\ldots\left\lceil D\left(x_{1}\times x_{2}\right)\right\rceil \times2\right]\\
D\left(x_{1}\times x_{2}\right) & \leftarrow D\left(x_{1}\times x_{2}\right)\cap\left[\left\lfloor D\left(x_{3}\right)\right\rfloor /2\ldots\left\lceil D\left(x_{3}\right)\right\rceil /2\right]
\end{cases}
\end{eqnarray*}

Unlike in the previous case, the box view propagator obtained from
composition with the non-idempotent propagator $\pi_{c}$ would not
achieve any pruning, since 
\begin{eqnarray*}
\beta_{f}^{+}\left(D\right) &  & \left[4\dots9\right]\times\left[9\dots15\right]\\
\pi_{c}^{\beta\beta}\left(\left[4\dots9\right]\times\left[9\dots15\right]\right) & = & \left[5\dots7\right]\times\left[9\dots15\right]\\
\widehat{\beta}_{f}\left(\left[5\dots7\right]\times\left[9\dots15\right],D\right)\cap D & = & \left[2\dots3\right]\times\left[2\dots3\right]\times\left[9\dots15\right]
\end{eqnarray*}

\end{example}
The reason for this may be explained as follows: not being idempotent,
the change in the domain of $D\left(x_{1}\times x_{2}\right)$ is
not immediately propagated back to $D\left(x_{3}\right)$ but instead
is approximated back to the initial value by the encapsulating view
functions, $\beta_{f}^{+}$ and $\widehat{\beta}_{f}$. Therefore,
the subsequent application of the box view propagator to have no effect. 

The above proposition explains the differences in propagation strength
between decompositions using auxiliary variables and view objects.
With variable decomposition, non-idempotent propagators for sub-expressions
will always act as idempotent propagators: sub-expressions are propagated
independently and will be added to the propagation queue until a fix
point is reached. In theory this means that decompositions using auxiliary
variables may attain stronger propagation. In practice these differences
do not seem to be very significant, and are compensated by other factors
as the experimental results presented later suggest. 

\selectlanguage{english}%

\section{\label{cha:Implementation-and-Experiments}Implementation}

Implementation of views for an arbitrary expression may be complex,
depending on the type of approximations the view is considering. For
the specific case of box view propagators, the code that implements
a view over an expression $e$ should be no more complex than the
code for a propagator of $e=z$, where $e$ is an expression involving
variables (i.e. not other expressions) and $z$ is a variable. For
example, implementing a box view for the expressions $e_{1}+e_{2}$,
or $e_{1}\times e_{2}$, where $e_{i}$ are expressions is very similar
to implementing the propagator for $x_{1}+x_{2}=x_{3}$, or $x_{1}\times x_{2}=x_{3}$,
where $x_{i}$ are variables. The simplicity of these implementations
is in fact a key advantage of box view propagators. The following
sections provide some examples of such implementations.

\subsection{Box view propagators in strongly typed programming languages\label{sec:impl-strongly-typed}}

Implementing our conceptual box view propagator in a strongly typed
programming language is possible if some sort of type polymorphism
support is available in the language. Most if not all popular strongly
typed programming languages have built in support for subtype polymorphism,
either by overloading or through the use of inheritance in the case
of object oriented programming languages. In addition, parametric
polymorphism has been introduced in some object oriented programming
languages such as C++, \foreignlanguage{british}{Java} and C\#. Parametric
polymorphism allows \foreignlanguage{british}{aggressive} compiler
optimizations, namely function code \foreignlanguage{british}{inlining}
(i.e. replacing function calls by the actual code), which has a significant
impact on performance as we will see later.

We have implemented box view propagators in C++. Since C++ supports
both subtype and parametric polymorphism, we were able to integrate
the two variants of our model within the constraint solver engine,
therefore obtaining a fair experimental platform.

\subsubsection{Subtype polymorphism}

Subtype polymorphism is available in C++ through the use of inheritance.
In this setting we need to define an abstract interface for box view
objects:
\begin{lstlisting}[basicstyle={\small},frame=lines,language={C++}]
class Box {
    virtual int getMin()=0;
    virtual int getMax()=0;
    virtual bool updMin(int i)=0; 
    virtual bool updMax(int i)=0; 
};
\end{lstlisting}

A box view object for a specific function implements the box view
object interface (for convenience the update methods return whether
the operation does not result in an empty box).
\begin{example}
The following class defines the subtype polymorphic box view object
for the addition of two box view objects. 
\begin{lstlisting}[basicstyle={\small},frame=lines,language={C++}]
class Add2 : Box {
    Add2(Box ax, Box ay) : x(ax),y(ay) {}
    virtual int getMin() { return x.getMin()+y.getMin(); }
    virtual int getMax() { return x.getMax()+y.getMax(); }
    virtual bool updMin(int i) 
    { return x.updMin(i-y.getMax()) and y.updMin(i-x.getMax());}
    virtual bool updMax(int i) 
    { return x.updMax(i-y.getMin()) and y.updMax(i-x.getMin());}
    Box x;
    Box y;
};
\end{lstlisting}

We should note that for the sake of simplicity we omit details on
the efficient copy and garbage collection of box view objects. Compiling
a given constraint into subtype polymorphic box view objects is straightforward
since in C++ expressions are evaluated bottom-up. Below are a set
of convenience functions which may be used to create subtype polymorphic
view box objects for a binary addition.
\end{example}
\begin{lstlisting}[basicstyle={\small},frame=lines,language={C++}]
Add2 add(Box x,Box y)
{ return Add2(x,y); }

Add2 operator+(Box x,Box y)
{ return Add2(x,y); }
\end{lstlisting}

The user may then create box view objects for arbitrary expressions
in C++ using a clean syntax:

\begin{lstlisting}[basicstyle={\small},frame=lines,language={C++}]
DomVar a,b,c;
a+b*c;
add(a,mul(b,c));
\end{lstlisting}

From what we could infer from the available documentation, \emph{constrained
expressions} introduced in ILOG Solver \cite{ILOG2003} correspond
to box view objects implemented using subtype polymorphism as shown
above.

\subsubsection{Parametric polymorphism\label{sub:Using-views-in-c++}}

The fact that the C++ compiler evaluates expressions bottom-up makes
the implementation of parametric polymorphic view objects slightly
more complex, since they need to be compiled top-down. The solution
we propose breaks the compilation algorithm in two phases. The first
phase creates a syntactic representation of the expression, called
a type parametric relation object, using the natural bottom-up evaluation
order intrinsic in the language. Type parametric relations capture
the data and the type of the objects and operations involved in the
constraint. After the full constraint is compiled, we use the obtained
relation object for instantiating the required view objects.

We will use templates for defining type parametric relations, since
this is the language mechanism available in C++ to support type parametric
polymorphism. The following template defines generic binary relations,
where ``Op'' is a type describing the operator, and ``X'' and
``Y'' are types of the operands.

\begin{lstlisting}[basicstyle={\small},frame=lines,language={C++}]
template<class Op,class X,class Y>
class Rel2 {
    Rel2(X x, Y y) : x(x),y(y) {}
    X x;
    Y y;
};
\end{lstlisting}

Since any expression may be transformed to a relation object with
a unique type, we can create view objects over arbitrary expressions
by defining templates over relation objects. 
\begin{example}
The following template defines the parametric polymorphic box view
object for the addition of two arbitrary objects:
\begin{lstlisting}[basicstyle={\small},frame=lines,language={C++}]
template<class X,class Y>
class Box<Rel2<Add,X,Y> > {
    Box(Rel2<Add,X,Y> r) : x(r.x),y(r.y) {}
    int getMin() { return x.getMin()+y.getMin(); }
    int getMax() { return x.getMax()+y.getMax(); }
    bool updMin(int i) 
    { return x.updMin(i-y.getMax()) and y.updMin(i-x.getMax());}
    bool updMax(int i) 
    { return x.updMax(i-y.getMin()) and y.updMax(i-x.getMin());}
    Box<X> x;
    Box<Y> y;
};
\end{lstlisting}

\end{example}
Parametric relation objects are created by a set of convenience functions,
such as: 
\begin{lstlisting}[basicstyle={\small},frame=lines,language={C++}]
template<class X,class Y>
Rel2<Add,X,Y> add(X x,Y y)
{ return Rel2<Add,X,Y>(x,y); }

template<class X,class Y>
Rel2<Add,X,Y> operator+(X x,Y y)
{ return Rel2<Add,X,Y>(x,y); }
\end{lstlisting}

Note that the above functions only create the relation object for
the expression, and not the corresponding view object. Creating the
view object is accomplished by providing the relation object to the
following function (which is basically a convenience function to avoid
specifying the parameter T):

\begin{lstlisting}[basicstyle={\small},frame=lines,language={C++}]
template<class T>
Box<T> box(T t) 
{ return Box<T>(t); }
\end{lstlisting}

The following code instantiates two parametric box view objects for
an expression using the above constructs.

\begin{lstlisting}[basicstyle={\small},frame=lines,language={C++}]
DomVar a,b,c;
box(a+b*c);
box(add(a,mul(b,c)));
\end{lstlisting}

\selectlanguage{british}%

\subsection{Incrementality}

\selectlanguage{english}%
It is important to note that incremental propagators (i.e. propagators
which maintain a state) may be used transparently with views. This
may be very useful in practice,  for example to model the bounds complete
distinct propagator over expressions, such as the main constraint
of the Golomb ruler problem (given in the next section), 
\[
\noun{{distinct}}\left(\left\{ x_{i}-x_{j}:1\leq j<i\leq m\right\} \right)
\]
There is also nothing preventing us from creating views that maintain
an internal state. Although we have implemented this for some expressions,
namely the \noun{Element} expression, it does not seem to be useful
for most $\beta\beta$ views that are cheap to evaluate. However,
it could certainly make a difference if using $\delta\delta$ views
for domain propagation. In summary, using views does not constrain
the propagator implementation model to be either incremental or non-incremental.

\subsection{Triggering}

Triggering is a well known method for decreasing the number of redundant
propagations during a fixpoint computation \cite{Schulte2004}. A
trigger may be seen as a condition for propagation - propagators are
known to be idempotent for the current domain until the condition
is true, i.e. when they may perform propagation. The way triggering
is used with views is not much different from its use in the variable
decomposition approach. In addition to the methods presented above,
each view object must provide methods for creating/deleting triggers
on the relevant events. This must be adapted to the events available
in the solver: for example, a trigger on both bounds of an expression
$x+y$ should map to the four bounds of $x$ and $y$, whereas a trigger
on the minimum of $-x$ maps to the maximum of $x$.

Given that each view object provide methods for creating and removing
triggers, moving triggers is also possible using views. Consider for
example a box view over the expression \noun{ifThenElse}$(C,T,F)$
which represents the box $T$ if $C$ is true, and $F$ if it is false.
Initially the box view object for this expression creates triggers
on $C$, $T$, and $F$ (note that inference can be made even if $C$
is not ground). When $C$ is instantiated it removes the trigger on
either $T$ or $F$. We should make two important remarks about creating/deleting
triggers. Firstly, at least in our current implementation, moving
triggers is not as efficient as in e.g. \cite{Gent2006}. This is
due to the fact that it is recursive on the structure of the expression
(consider for example creating a trigger on an expression containing
deeply nested expressions). This is tantamount with other aspects
of views - a specialized propagator for the constraint will always
have better runtime performance. Secondly, view objects must always
check beforehand if it is safe to move a trigger. In our example,
the box view object must check if $C$ is ground (i.e. if getMin()=getMax())
before deleting the relevant triggers, even if it has updated the
domain of $C$ itself. This has to do with (non) persistence of update
operations explained in the following section.

\subsection{Persistent operations and idempotency checking}

An important feature of box view objects is that operations are not
guaranteed to be persistent. Consider for example the constraint $c=\left[X\neq k\right]$,
where $X$ is an expression and $k$ is a constant. The pseudocode
for a bounds consistent propagator for this constraint is as follows:

\begin{lstlisting}[basicstyle={\small},frame={lines},language={C++},numbers=left]
bool propagate(Box<int> X, int k) {
    if (X.min()==k)
        return X.updateMin(k+1);
    if (X.max()==k)
        return X.updateMax(k-1);
    return true;
}
\end{lstlisting}
This propagator may be used for propagating several constraints by
specifying different expressions for $X$, for instance the constraints
$c_{1}=\left[x_{1}\neq k\right]$ and $c_{2}=\left[y_{1}+y_{2}\neq k\right]$
where $X$ is respectively $x_{1}$ and $y_{1}+y_{2}$. While lines
3 and 5 guarantee that the domain of expression $X$ will be different
from $k$ when propagating constraint $c_{1}$, this is not guaranteed
when propagating constraint $c_{2}$ (consider for example $D\left(y_{1}\right)=D\left(y_{2}\right)=\left\{ 1,2\right\} $
and $k=4$). In this simple propagator the fact that some update operations
are non-persistent is not important: the propagator will be scheduled
again when either the minimum or maximum of $X$ changes, eventually
leading to a persistent update. In general, view based propagators
must be designed to take non-persistent update operations into consideration.
The class of propagators that check for idempotency deserves special
attention. 

Idempotency checking is a technique for decreasing the number of redundant
propagations during solving. The main idea is to avoid scheduling
propagators that report idempotency for a given domain since no further
pruning would be achieved. A naive (but incorrect) implementation
of this optimization on the propagator above would be as follows:

\begin{lstlisting}[basicstyle={\small},frame={lines},language={C++}]
status propagate(Box<int> X, int k) {
    if (X.min()==k)
        return (X.updateMin(k+1))?idempotent:failed;
    if (X.max()==k)
        return (X.updateMax(k-1))?idempotent:failed;
    return suspend;
}
\end{lstlisting}
In the above pseudocode a propagator can be in one of three states:
failed (i.e. inconsistent for the current domain), idempotent, or
suspended (i.e. neither failed nor idempotent). Due to non-persistent
operations for some instantiations of $X$ this propagator may report
``idempotent'' when in fact it is not. Instead, a correct implementation
of this optimization must not rely on persistent update operations:
\begin{lstlisting}[basicstyle={\small},frame={lines},language={C++}]
status propagate(Box<int> X, int k) {
    if (X.min()==k)
    {    
        if (X.updateMin(k+1))
            return (X.min()>k)?idempotent:suspend;
        else
            return failed;
    }
    if (X.max()==k)
        (similar)
    return suspend;
}
\end{lstlisting}
In summary, checking for idempotent view-based propagators is possible
but must be carefully designed in order not to rely on persistent
update operations as shown above.

\selectlanguage{british}%

\subsection{\label{cha:Complexity_analysis} Complexity analysis}

The adoption of views avoids introducing auxiliary variables and propagators
for every subexpression. Conceptually, a view object over an expression
serves the same purpose as the auxiliary variable introduced for that
expression: to expose its domain. However these models have distinct
operational tradeoffs. To illustrate this let us focus on an arithmetic
constraint involving $n$ variables with uniform domain size $d$,
with an unbalanced syntax tree, i.e. where each operator in the expression
involves at least one variable. Figure \ref{fig:Expression-AST-1}
shows a fragment of the expression syntax tree. We will refer to the
decomposition model using auxiliary variables as \noun{Vars}, subtype
polymorphic view as \noun{SViews}, and parametric polymorphic views
as \noun{PViews}.

\begin{figure}
\hfill{}\begin{tikzpicture}[scale=0.7]
\path (2,4) node (nn_1) {$n_{n-1}$} --
      (0,2) node (n2) {$n_2$} --
      (-1,1) node (n1) {$n_1$} --
      (-2,0) node (l1) {$l_1$};
\path (n1) -- (0,0) node (l2) {$l_2$};
\path (n2) -- (1,1) node (l3) {$l_3$};
\path (nn_1) -- (3,3) node (ln) {$l_n$};
\draw [dashed] (nn_1) -- (n2);
\draw (n2) -- (n1) -- (l1);
\draw (n1) -- (l2);
\draw (n2) -- (l3);
\draw (nn_1) -- (ln);
\end{tikzpicture}\hfill{}

\caption{\label{fig:Expression-AST-1}An unbalanced expression syntax tree.
The internal nodes $n_{1}\ldots n_{n-1}$ represent operators and
leafs $l_{1}\ldots l_{n}$ represent variables.}
\end{figure}
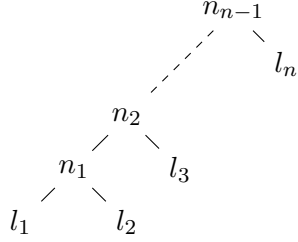

\subsubsection{Memory\label{sub:views-memory}}

A box view object can be designed to expose just the subset of the
expression's domain required for the view's client (e.g. the bounds
of the expression). In contrast, a variable maintains the domain of
the expression, possibly containing regions that will always be ignored
for propagation. For an expression containing $n-1$ operators (fig.~\ref{fig:Expression-AST-1}),
the memory overhead of the \noun{Vars} model is in $O\left(nD\right)$,
where $D$ is the size of the largest domain of an auxiliary variable.
In practice, although $D$ may be as large as $d^{n-1}$, many solvers
\cite{GecodeTeam2006,Gent2006b,ILOG2003} use intervals to store the
domains of the auxiliary variable (i.e. $D=2$), thus eliminating
this problem.

\subsubsection{Runtime\label{sub:views-runtime}}

The analysis of runtime complexity focuses on the number of propagator
executions, function calls, arithmetic operations performed and updates
of variable bounds, for a single propagation of the expression. We
consider two worst-case situations: a) when propagation (up to the
root) is due to a change in the domain of a leaf variable, and b)
when propagation to leaf variables is caused by a change in the domain
of the root. These situations correspond, respectively, to accessing
and updating the bounds of the expression.
\begin{description}
\item [{a)}] In the \noun{Vars} model $O\left(n\right)$ propagators may
execute, forcing changes in the bounds of $O\left(n\right)$ (auxiliary)
variables with $O\left(n\right)$ operations. In both view models,
an update of a leaf variable domain causes one single propagator to
execute and evaluate the full expression. Such evaluation requires
$O\left(n\right)$ function calls and $O\left(n\right)$ operations%
\footnote{A single update of the expression requires the evaluation of the full
tree, possibly involving $n^{2}$ total operations (as well as $n^{2}$
function calls in case of the \noun{SViews} model), if subexpressions
are evaluated multiple times recursively. It is possible to decrease
this number to $n$ if the evaluation of the subexpressions is cached
during the execution of the propagator. This may be done efficiently
with box views since the cached data is simply the two bounds of the
subexpression that do not have to be backtrackable.%
}. 
\item [{b)}] In the \noun{Vars} model $O\left(n\right)$ propagators may
execute, forcing changes in the bounds of up to $O\left(2n\right)$
variables ($n$ leaf variables and $n-2$ auxiliary variables), requiring
up to $O\left(n\right)$ operations. Both view models force $O\left(n\right)$
updates of bounds of the (leaf) variables, through $O\left(n\right)$
function calls and $O\left(n\right)$ operations, but no extra propagators
will execute. 
\end{description}
Although all the models present the same worst-case propagation complexity
of $O\left(n^{2}d\right)$, breaking down the costs shows that the
view models perform fewer calls to propagators and fewer variable
updates (there are no auxiliary variable) than the \noun{Vars} model.
The fundamental operational difference between the \noun{Vars} model
and both view models is that a view object computes its domain on
demand, that is, it will never update its domain before needed by
the view\textquoteright{}s client, in contrast with the \noun{Vars}
model, where additional propagators will be posted to prune the domains
of the auxiliary variables. Among all contributions, managing the
extra propagators should involve the most significant costs. On the
other hand, the view models require extra operations and function
calls. Overall, and despite the same worst-case complexity we expect
that the VARS models, that may suffer from the execution of more propagators,
are out-performed by the \noun{Views} models, specially by the \noun{PViews}
where the compiler is often able to avoid the overhead of function
calls using inlining. But these expectations can only be adequately
assessed by experimentation, so in the next section we present results
obtained in a comprehensive set of experiments.\selectlanguage{british}%

\section{Experimental Results\label{sec:Experiments}}

In this section we evaluate the performance of the decomposition methods
described in the previous sections on a set of benchmarks.

\subsection{Experiments}

Specifically we are interested in comparing the following models.

\subsubsection{Models}
\begin{description}
\item [{~}] \textbf{\textsc{Vars}}. This is the classical method for
decomposing constraints into primitive propagators introducing one
auxiliary variable for each subexpression.
\item [{~}] \textbf{\textsc{Vars+Global}}. This model is similar to the
previous but uses global constraints for lowering the number of auxiliary
variables. Only a subset of problems support this decomposition in
which case we will specifically mention which global constraints are
used.
\item [{~}] \textbf{\textsc{PViews}}. The model that implements the decomposition
based on parametric polymorphic view objects.
\item [{~}] \textbf{\textsc{SViews}}. The decomposition based on subtype
polymorphic view objects.
\item [{~}] \textbf{\textsc{Views+Global}}. Like the \noun{Vars+Global}
model, this model uses a combination of some type of views and a global
constraint propagator.
\end{description}
All the above decomposition models were implemented in CaSPER \cite{Correia2008,Correia2005}.
Additionally, we also implemented the first two in Gecode \cite{GecodeTeam2006}
denoted \noun{Gecode-Vars} and \noun{Gecode-Vars+Global} respectively.
Comparing to the Gecode solver assesses the competitiveness of CaSPER
as a whole, in order to clarify that the performance of using our
method for propagation compared to using auxiliary variables is not
due to an inefficient implementation of the latter.

\subsubsection{Problems\label{sub:Problems}}

The set of benchmarks covers a total of 22 instances from 6 different
problems. Before we present them in detail, we should make a few general
considerations. 

For each given instance we used the same labeling heuristics (or no
heuristics at all) for testing the above models. This means that,
for each instance, the solvers resulting from the implementation of
the above models explore exactly the same search space, unless the
decompositions have different propagation strength, which may occur
as we have already seen.

In the following, we will mostly focus on the decomposable constraints
for which our models apply. When describing the model, we may choose
to ignore other necessary, redundant, or symmetry breaking constraints
that we used in our implementations. These were kept constant across
all implementations of the above models for each benchmark and therefore
do not influence our conclusions. For additional information, we provide
references to detailed descriptions of the problems in the online
constraint programming benchmark database CSPLib \cite{Gent1999}.
The source code for all the implementations can be obtained from the
first author upon request.

\paragraph*{Systems of linear equations.}

This experiment consists in solving a system of linear equations.
Linear equations are usually integrated in Constraint Programming
using a specific global constraint propagator. The goal of the experiment
is therefore to assess the overhead of decomposing expressions using
the presented models compared to a decomposition which uses a special
purpose algorithm, i.e. a global constraint. 

Each system of linear equations is described by a tuple $\left\langle n,d,c,a,s|u\right\rangle $
where $n$ is the number of variables in the problem, $d$ is the
uniform domain size, $c$ is the number of linear equations, $a$
is the number of terms in each equation, and the last term denotes
if the problem is (s)atisfiable or (u)nsatisfiable. Each problem is
defined by

\[
\bigwedge_{i=1}^{c}\sum_{v\in p\left(i\right)}v=t
\]
where $p\left(i\right)$ is a function returning a combination of
$a$ variables for the equation $i$, selected randomly from the full
set of $C_{a}^{n}$ possible combinations. The independent term $t$
in each equation was selected randomly with a uniform probability
from the interval $\left[a\ldots a\times d\right]$. Different random
seeds were experimented in order to generate difficult instances. 

The $a$-ary sum constraints were decomposed using binary sums implementing
a subset of the previously described models, namely \noun{Vars}, \noun{SViews},
and \noun{PViews}. Model \noun{Vars+Global} used global constraint
propagators for the $a$-ary sums (in this case no auxiliary variables
are required).

\paragraph*{Systems of nonlinear equations.}

The second experiment considers systems of nonlinear equations. These
problems arise often in practice, and since the decomposition to special
purpose propagators is not so direct as in the previous case, it provides
a realistic opportunity to apply the previously discussed models. 

A system of nonlinear equations is described by a tuple $\left\langle n,d,c,a_{1},a_{2},s|u\right\rangle $
where $a_{1}$ is the number of terms in each equation, each term
is composed of a product of $a_{2}$ factors, and all remaining variables
have the same meaning as before. Each system of nonlinear equations
is formally defined as:

\[
\bigwedge_{i=1}^{c}\sum_{j=1}^{a_{1}}\prod_{v\in p\left(i,j\right)}v=t
\]
where $p\left(i,j\right)$ is a function returning a combination of
$a_{2}$ variables for the term $j$ of equation $i$, selected randomly
from the full set of $C_{a_{2}}^{n}$ possible combinations. 

The tested models consist of the decomposition into binary sums and
products using auxiliary variables exclusively (\noun{Vars}) and using
view models (\noun{SViews}, \noun{PViews}). We also tested two models
where each product is decomposed using either auxiliary variables
or views, projected to a variable $x_{i}$, and a sum propagator is
used to enforce $\sum_{i=1}^{a_{1}}x_{i}$ for each equation (\noun{Vars+Global}
and \noun{PViews+Global} respectively).

\paragraph*{Social golfers (prob10 in CSPLib).}

The Social golfers problem consists in scheduling a golf tournament.
The golf tournament lasts for a specified number of weeks $w$, organizing
$g$ games each week, each game involving $s$ players. There is therefore
a total of $g\times s$ players participating in the tournament. The
goal is to come up with a schedule where each pair of golfers plays
in the same group at most once.

This problem may be solved efficiently in Constraint Programming using
a 3-dimensional matrix $x$ of $w\times g\times s$ integer domain
variables, where each variable identifies a golf player in the tournament.
For two groups of players $G_{1}$, $G_{2}$, the \noun{meetOnce}
constraint ensures that any pair of players in one group does not
meet in the other group,
\begin{eqnarray*}
\noun{{meetOnce}}\left(G_{1},G_{2}\right) & = & \left[\sum_{x\in G_{1},y\in G_{1}}\left(x=y\right)\leq1\right]
\end{eqnarray*}
This constraint is then used to impose that each pair of players meets
at most once during the entire tournament,
\[
\bigwedge_{1\leq w_{i}<w_{j}\leq w}\noun{{meetOnce}}\left(\begin{array}{c}
\left\{ x_{w_{i},g_{i},s_{i}}:1\leq g_{i}\leq g,1\leq s_{i}\leq s\right\} ,\\
\left\{ x_{w_{j},g_{j},s_{j}}:1\leq g_{j}\leq g,1\leq s_{j}\leq s\right\} 
\end{array}\right)
\]

We tested a subset of our models for propagating the \noun{meetOnce}
constraint, namely \noun{PViews+Global}, \noun{SViews+Global}, and
\noun{Vars+Global}. View based models implement the \noun{meetOnce}
constraint using the above expression directly using a global propagator
for the sum constraint:
\begin{eqnarray*}
\noun{{meetOnce}}\left(G_{1},G_{2}\right) & = & \left[\noun{{sum}}\left(\noun{{all}}\left(x\in G_{1},y\in G_{2},x=y\right)\right)\leq1\right]
\end{eqnarray*}
The \noun{all} predicate is used for aggregating a set of expressions
from the instantiation of a template expression (e.g. $x=y$) over
a set of possible values (e.g. $x\in G_{1},y\in G_{2}$).

The \noun{Vars+Global} model implements the traditional decomposition
using a set $b$ of $s^{2}$ auxiliary boolean domain variables, 
\begin{eqnarray}
\noun{{meetOnce}}\left(G_{1},G_{2}\right) & = & \bigwedge_{x\in G_{1},y\in G_{2}}\left[b_{i}=\left(x=y\right)\right]\label{eq:golfers-meetonce-eq1}\\
 & \wedge & \left[\sum_{i=1}^{s^{2}}b_{i}\leq1\right]\label{eq:golfers-meetonce-eq2}
\end{eqnarray}
In this case we used the (reified) equality propagator for each equation
in the conjunction of eq.~\vref{eq:golfers-meetonce-eq1}, and a
sum global propagator for eq.~\vref{eq:golfers-meetonce-eq2}.

\paragraph*{Golomb ruler (prob6 in CSPLib).}

A Golomb ruler of $m$ marks and length $x_{m}$ is a set of $m$
integers, 
\[
0=x_{1}<x_{2}<\ldots<x_{m}
\]
such that the $m(m-1)/2$ differences $x_{i}-x_{j}$, $1\leq j<i\leq m$
are distinct. The Golomb ruler problem is an optimization problem,
where the goal is to find the smallest possible Golomb ruler with
a given number of marks.

This problem makes use of a constraint of the form
\[
\noun{{distinct}}\left(\left\{ x_{i}-x_{j}:1\leq j<i\leq m\right\} \right)
\]
enforcing that the pairwise differences $x_{i}-x_{j}$ are all distinct.
The classical decomposition for this constraint (\noun{Vars+Global})
introduces one auxiliary variable for each pairwise difference, and
makes use of the distinct global propagator for the \noun{distinct}
constraint,
\begin{eqnarray*}
 & \bigwedge_{1\leq i<j\leq m} & \left[a_{i,j}=x_{i}-x_{j}\right]\\
 & \wedge & \left[\noun{{distinct}}\left(a\right)\right]
\end{eqnarray*}

Using views avoids introducing the set of auxiliary variables $a$.
Instead, the constraint is used directly as follows,
\[
\noun{{distinct}}\left(\noun{{all}}\left(1\leq i\leq m,1\leq j\leq m,j<i,x_{i}-x_{j}\right)\right)
\]
and enforced with the bounds complete propagator introduced by \cite{Lopez-ortiz2003}.
In our benchmarks we solved the decision version of this problem,
i.e. we provided the size of the ruler $x_{m}$ as a parameter of
the problem, and asked for a ruler satisfying the constraints.

\paragraph*{Low autocorrelation binary sequences (prob5 in CSPLib).}

The goal is to construct a binary sequence $S=x_{1},\ldots,x_{n}$
of length $n$, where $D\left(x_{i}\right)=\left\{ -1,1\right\} $,
$1\leq i\leq n$, that minimizes the autocorrelations between bits,
i.e. that minimizes the following expression,
\begin{equation}
m=\sum_{i=1}^{n-1}\left(\sum_{j=1}^{n-i-1}x_{j}\times x_{j+i+1}\right)^{2}\label{eq:labs-m}
\end{equation}
The \noun{Vars+Global} model decomposes the above expression using
three sets of auxiliary variables $a$, $b$, and $c$, and sum constraints
as follows, 
\begin{eqnarray*}
 & \bigwedge_{i=1}^{n-1}\bigwedge_{j=1}^{n-i-1} & \left[x_{j}\times x_{j+i+1}=a_{i,j}\right]\\
 & \bigwedge_{i=1}^{n-1} & \left[\sum_{j=1}^{n-i-1}a_{i,j}=b_{i}\right]\\
 & \bigwedge_{i=1}^{n-1} & \left[b_{i}^{2}=c_{i}\right]\\
 & \wedge & \left[m=\sum_{i=1}^{n-1}c_{i}\right]
\end{eqnarray*}

The \noun{PViews, }and\noun{ SViews} models implements expression
\ref{eq:labs-m} directly, without introducing any auxiliary variable.

\paragraph*{Fixed-length error correcting codes (prob36 in CSPLib).}

This problem involves generating a set of strings from a given alphabet
which satisfy a pairwise minimum distance. Each instance is defined
by a tuple $\left\langle a,n,l,d\right\rangle $ where $a$ is the
alphabet size, $n$ is the number of strings, $l$ is the string length,
and $d$ is the minimum distance allowed between any two strings.
For measuring the distance between two strings we have used the Hamming
distance on two instances and the Lee distance on the other two. For
two arbitrary strings $x$, $y$ of length $l$, these measures are
defined as follows,
\begin{eqnarray}
\noun{{Hamming}}\left(x,y\right) & = & \sum_{i=1}^{l}\left(x_{i}\neq y_{i}\right)\nonumber \\
\noun{{Lee}}\left(x,y\right) & = & \sum_{i=1}^{l}\min\left(\left|x_{i}-y_{i}\right|,a-\left|x_{i}-y_{i}\right|\right)\label{eq:lee-distance}
\end{eqnarray}
This problem was modeled by a matrix $x$ of $n\times l$ integer
domain variables, where each variable $x_{i,j}$ can take a value
in $1\ldots a$ corresponding to the symbol of string $i$ in position
$j$. Then, distance constraints are imposed between each pair of
strings,
\[
\bigwedge_{1\leq i_{1}<i_{2}\leq n}\noun{{distance}}\left(\left\{ x_{i_{1},j}:1\leq j\leq l\right\} ,\left\{ x_{i_{2},j}:1\leq j\leq l\right\} \right)\geq d
\]
The \noun{Vars+Global} model decomposes distance constraints using
auxiliary variables and sum constraints. Note that in the case of
the Lee distance, a total of $4l$ auxiliary variables are introduced
for each distance constraint. The \noun{SViews}, and \noun{PViews}
models implement both distance functions without any extra variables.

\subsection{Setup}

The experiments were compiled with the gcc-4.5.3 C++ compiler and
executed on an Intel Core i7 @ 3.39GHz running Mac OS X 10.7.4. The
versions of the CaSPER and Gecode solvers were the most recently available
at the time of these experiments, respectively version 1.0.0rc2 and
version 3.7.3. Each benchmark was repeated ten times and then kept
repeating until the standard deviation of the runtime was below $2\%$
of the average time. The minimum runtime was then used. The source
code for all experiments can be made available upon request (please
email the first author).

\subsection{Discussion\label{sec:Discussion}}

The results of all benchmarks are summarized in tables \ref{tab:pviewsVsSviews}
to \ref{tab:pviewsVsGecode+global} (more detailed in \cite{Corr10})
from which we may draw the following conclusions.

\subsubsection{Type parametric views versus Subtype polymorphic views}

Recall from section \vref{sec:impl-strongly-typed} that solvers using
type parametric view objects are able to avoid a number of function
calls due to code inlining optimizations. Table \ref{tab:pviewsVsSviews}
shows how this optimization improves performance in practice. Overall
the speed-up of PViews wrt. SViews is $50\%$ (i.e. $1/0.67-1$).
In particular for problems involving a large number of subexpressions
the speed-up can reach up to $300\%$. Given that type parametric
views are consistently better, we choose to use this model exclusively
on the remaining experiments.

\begin{table}
\hfill{}%
\begin{tabular}{|c|c|c|c|c|}
\hline 
 & mean & stddev & min & max\tabularnewline
\hline 
\hline 
Systems of linear equations & 0.44 & 1.87 & 0.24 & 0.97\tabularnewline
\hline 
Systems of nonlinear equations & 0.74 & 1.02 & 0.72 & 0.75\tabularnewline
\hline 
Social golfers & 0.70 & 1.03 & 0.68 & 0.72\tabularnewline
\hline 
Golomb ruler & 0.98 & 1.02 & 0.95 & 1.00\tabularnewline
\hline 
Low autocorrelation binary sequences & 0.92 & 1.01 & 0.91 & 0.92\tabularnewline
\hline 
Fixed-length error correcting codes & 0.56 & 1.44 & 0.39 & 0.77\tabularnewline
\hline 
\hline 
All & 0.67 & 1.48 & 0.24 & 1.00\tabularnewline
\hline 
\end{tabular}\hfill{}

\caption{\label{tab:pviewsVsSviews}Geometric mean, standard deviation, minimum
and maximum of the ratios between the runtimes of the solver implementing
the \noun{PViews} model and the solver implementing the \noun{SViews}
model, on all benchmarks.}
\end{table}

\subsubsection{Auxiliary variables versus Type parametric views}

Table \ref{tab:varsVsViews} compares the runtime of the best model
using auxiliary variables, i.e. either \noun{Vars} or \noun{Vars+Global},
with the runtime of the best model that uses type parametric views,
i.e. either \noun{PViews} or \noun{PViews+Global}. View objects do
not intend to be a replacement for global constraint propagators,
and therefore this table shows how much the runtime of a constraint
program may be improved when using the best available tools. 

Before we take a global view on the results in this table, let us
focus on the special case of the benchmark involving systems of linear
equations. We recall that this benchmark should not be considered
as part of a realistic application of views or auxiliary variables
since it may be modeled using a global constraint propagator exclusively.
However, modeling the global sum constraint using type parametric
views over binary sums was only 10\% worse on average, which is nevertheless
remarkable.

For all other benchmarks, using type parametric views instead of auxiliary
variables was consistently better, approximately twice as fast on
(geometric) average, and always more than 33\% faster. An interesting
particular case are the two instances of the ``Fixed length error
correcting codes'' problem using the Lee distance. These instances
are in fact the only ones for which decomposing using auxiliary variables
could be recommended. This is because there is a subexpression which
occurs twice in the expression, and therefore can be represented by
the same auxiliary variable (see equation \ref{eq:lee-distance}),
possibly leading to a smaller search tree. Even without this optimization,
the solver using type parametric views was almost twice as fast on
average on these instances.

Regarding propagation, even if using auxiliary variables may sometimes
lead to smaller search trees (due to proposition \ref{prop:prop-strong-idemp-1}),
the difference was not significant in our experiments. In fact, for
those instances where using auxiliary variables increases propagation
strength compared to views, the discrepancy in the number of fails
was only of $6\%$ on average, and never more than $20\%$ (table
\ref{tab:FailsVarsVsViews}).

\begin{table}
\hfill{}%
\begin{tabular}{|c|c|c|c|c|}
\hline 
 & mean & stddev & min & max\tabularnewline
\hline 
\hline 
Systems of linear equations & 1.10 & 1.14 & 0.96 & 1.24\tabularnewline
\hline 
Systems of nonlinear equations & 0.38 & 1.08 & 0.34 & 0.41\tabularnewline
\hline 
Social golfers & 0.28 & 1.29 & 0.22 & 0.36\tabularnewline
\hline 
Golomb ruler & 0.75 & 1.07 & 0.70 & 0.80\tabularnewline
\hline 
Low autocorrelation binary sequences & 0.24 & 1.01 & 0.24 & 0.24\tabularnewline
\hline 
Fixed-length error correcting codes & 0.60 & 1.30 & 0.42 & 0.79\tabularnewline
\hline 
\hline 
All & 0.51 & 1.73 & 0.22 & 1.24\tabularnewline
\hline 
All except linear & 0.43 & 1.55 & 0.22 & 0.80\tabularnewline
\hline 
\end{tabular}\hfill{}

\caption{\label{tab:varsVsViews}Geometric mean, standard deviation, minimum
and maximum of the ratios between the runtimes of best performing
model using views and the best performing model using auxiliary variables,
on all benchmarks.}
\end{table}

\begin{table}
\hfill{}%
\begin{tabular}{|c|c|c|c|c|}
\hline 
 & mean & stddev & min & max\tabularnewline
\hline 
\hline 
Systems of nonlinear equations & 1.06 & 1.08 & 1.01 & 1.16\tabularnewline
\hline 
Golomb ruler & 1.00 & 1.00 & 1.00 & 1.01\tabularnewline
\hline 
Fixed-length error correcting codes & 1.20 & 1.00 & 1.20 & 1.20\tabularnewline
\hline 
\hline 
All & 1.06 & 1.08 & 1.00 & 1.20\tabularnewline
\hline 
\end{tabular}\hfill{}

\caption{\label{tab:FailsVarsVsViews}Geometric mean, standard deviation, minimum
and maximum of the ratios between the number of fails of the best
performing solver using views and the best performing solver using
auxiliary variables, on all instances of each problem where the number
of fails differ.}
\end{table}

\subsubsection{Competitiveness}

Modeling decomposable constraints using type parametric views makes
CaSPER competitive with the state-of-the-art Gecode solver. This may
be seen by comparing the results presented in table \ref{tab:globalVsGecode+global}
and table \ref{tab:pviewsVsGecode+global}. In the first table we
compare the runtimes obtained by running the same model on both solvers,
i.e. \noun{Vars+Global} and \noun{Gecode-Vars+Global}. The second
table compares the \noun{PViews} model against \noun{Gecode-Vars+Global,
}which are the\noun{ }best models that can be implemented in both
platforms using the available modeling primitives. While CaSPER is
worse in all but one problem when using auxiliary variables and global
propagators, it becomes faster when using type parametric views. We
believe that the discrepancy observed in the ``Fixed length error
correcting codes'' benchmark is related to aspects of the architecture
of both solvers which are orthogonal to the tested models.

\begin{table}
\hfill{}%
\begin{tabular}{|c|c|c|c|c|}
\hline 
 & mean & stddev & min & max\tabularnewline
\hline 
\hline 
Systems of linear equations & 1.23 & 1.41 & 0.73 & 1.50\tabularnewline
\hline 
Systems of nonlinear equations & 2.06 & 1.06 & 1.90 & 2.16\tabularnewline
\hline 
Social golfers & 2.00 & 1.28 & 1.51 & 2.38\tabularnewline
\hline 
Golomb ruler & 1.68 & 1.29 & 1.25 & 2.01\tabularnewline
\hline 
Low autocorrelation binary sequences & 1.84 & 1.01 & 1.83 & 1.85\tabularnewline
\hline 
Fixed-length error correcting codes & 0.11 & 2.23 & 0.04 & 0.21\tabularnewline
\hline 
\hline 
All & 1.01 & 3.30 & 0.04 & 2.38\tabularnewline
\hline 
All except fixed-length error correcting codes & 1.72 & 1.33 & 0.73 & 2.38\tabularnewline
\hline 
\end{tabular}\hfill{}

\caption{\label{tab:globalVsGecode+global}Geometric mean, standard deviation,
minimum and maximum of the ratios between the runtimes of the CaSPER
solver implementing the \noun{Vars+Global} model and the Gecode solver
implementing the \noun{Vars+Global} model, on all benchmarks.}
\end{table}

\begin{table}
\hfill{}%
\begin{tabular}{|c|c|c|c|c|}
\hline 
 & mean & stddev & min & max\tabularnewline
\hline 
\hline 
Systems of linear equations & 1.35 & 1.56 & 0.71 & 1.86\tabularnewline
\hline 
Systems of nonlinear equations & 0.78 & 1.08 & 0.72 & 0.86\tabularnewline
\hline 
Social golfers & 0.57 & 1.42 & 0.45 & 0.85\tabularnewline
\hline 
Golomb ruler & 1.26 & 1.24 & 1.00 & 1.53\tabularnewline
\hline 
Low autocorrelation binary sequences & 0.44 & 1.02 & 0.44 & 0.45\tabularnewline
\hline 
Fixed-length error correcting codes & 0.06 & 2.26 & 0.03 & 0.13\tabularnewline
\hline 
\hline 
All & 0.52 & 3.23 & 0.03 & 1.86\tabularnewline
\hline 
All except fixed-length error correcting codes & 0.85 & 1.61 & 0.44 & 1.86\tabularnewline
\hline 
\end{tabular}\hfill{}

\caption{\label{tab:pviewsVsGecode+global}Geometric mean, standard deviation,
minimum and maximum of the ratios between the runtimes of the CaSPER
solver implementing the \noun{PViews} model and the Gecode solver
implementing the \noun{Vars+Global} model, on all benchmarks.}
\end{table}

In summary, we have seen how box view objects may be implemented using
several language paradigms, with a focus on strongly typed languages,
namely C++. Decomposition models based on auxiliary variables, subtype
polymorphic views, and type parametric views were implemented for
a number of well known benchmarks, and the results were discussed.
We observed that type parametric views are clearly more efficient
than models resulting from the other decomposition/compilation methods
for all benchmarks. Moreover, we have seen that this technique improves
the performance of CaSPER to the point of being competitive with Gecode,
which is regarded as one of the most efficient solvers available.

\selectlanguage{british}%

\subsubsection{\label{sub:Completeness-revisited}Monitoring Execution}

The theoretical discussion in the previous section hints at the reasons
why the view models, in particular the \noun{PViews} model, achieve
better performance than the \noun{Vars} model. We checked whether
these hints were confirmed in the experiments described earlier. To
do so, we monitored the execution of many instances of the problems
and consistently obtained results similar to those reported below.
Table \ref{tab:monitor-1} shows the results obtained with two instances
of the problem involving systems of nonlinear equations (see \ref{sub:Problems}).
First, we note that in the first (satisfiable) instance, the same
search trees were explored by all models, whereas in the second (unsatisfiable)
the \noun{Vars} model explores a slightly smaller tree (1\% less failures).
As expected, the better performance of both view models (about 2 times
faster) is due to the lower number of propagator executions as well
as domain updates (one order of magnitude less than in the \noun{Vars}
model). Moreover, \noun{PViews} improves (20-30 \%) on the \noun{SViews}
model because of its better inlining, although the compiler is not
able to inline all the function calls. 

\begin{table}
\hfill{}%
\begin{tabular}{|>{\raggedright}m{1cm}|c|c|c|c|c|c|c|}
\hline 
\multicolumn{2}{|c|}{Non-Linear} & p & t & f & u & calls & op\tabularnewline
\hline 
\hline 
\multirow{3}{1cm}{20-20-10-4-2-s} & \noun{Vars} & 1.9E+8 & 15.7 & 651,189 & 1.2E+8 & 0.0E+0 & 1.1E+9\tabularnewline
\cline{2-8} 
 & \noun{SViews} & 2.4E+7 & 8.8 & 651,189 & 2.0E+7 & 1.5E+9 & 1.4E+9\tabularnewline
\cline{2-8} 
 & \noun{PViews} & 2.4E+7 & 6.4 & 651,189 & 2.0E+7 & 9.7E+8 & 1.4E+9\tabularnewline
\hline 
\multirow{3}{1cm}{50-5-20-4-3-u} & \noun{Vars} & 1.6E+8 & 19.3 & 342,311 & 1.0E+8 & 0.0E+0 & 9.6E+8\tabularnewline
\cline{2-8} 
 & \noun{SViews} & 1.2E+7 & 8.9 & 346,762 & 1.0E+7 & 1.5E+9 & 1.6E+9\tabularnewline
\cline{2-8} 
 & \noun{PViews} & 1.2E+7 & 6.5 & 346,762 & 1.0E+7 & 5.2E+8 & 1.6E+9\tabularnewline
\hline 
\end{tabular}\hfill{}

\caption{\label{tab:monitor-1}Number of propagator executions (p), seconds
(t), fails (f), domain updates (u), calls to view object methods (calls),
and arithmetic operations (op) when solving two instances of the systems
of nonlinear equations problem. }
\end{table}

A similar picture is obtained with three instances of the Golomb problem
presented in table \ref{tab:monitor-2}. Again, despite exploring
a slightly smaller search space (<1\% less failures), the performance
of the \noun{Vars} model is penalized (about 33\% slow down) by the
larger number of propagator executions as well as domain updates,
both again one order of magnitude higher than in the view models.
Within these, \noun{PViews} performs slightly better due to function
inlining  (in this case, there are fewer function calls than the previous
example even if all function calls were inlined).

\begin{table}
\hfill{}%
\begin{tabular}{|c|c|c|c|c|c|c|c|}
\hline 
\multicolumn{2}{|c|}{Golomb} & p & t & f & u & calls & op\tabularnewline
\hline 
\hline 
\multirow{3}{*}{10} & \noun{Vars} & 3.64E+5 & 0.05 & 1,703 & 2.2E+5 & 0 & 1.9E+6\tabularnewline
\cline{2-8} 
 & \noun{SViews} & 3.72E+4 & 0.04 & 1,707 & 3.0E+4 & 2.3E+5 & 1.7E+6\tabularnewline
\cline{2-8} 
 & \noun{PViews} & 3.72E+4 & 0.04 & 1,707 & 3.0E+4 & 0 & 1.7E+6\tabularnewline
\hline 
\multirow{3}{*}{11} & \noun{Vars} & 2.06E+6 & 0.3 & 7,007 & 1.3E+6 & 0 & 1.1E+7\tabularnewline
\cline{2-8} 
 & \noun{SViews} & 1.84E+5 & 0.22 & 7,063 & 1.5E+5 & 1.2E+6 & 1.1E+7\tabularnewline
\cline{2-8} 
 & \noun{PViews} & 1.84E+5 & 0.21 & 7,063 & 1.5E+5 & 0 & 1.1E+7\tabularnewline
\hline 
\multirow{3}{*}{12} & \noun{Vars} & 1.14E+8 & 17.7 & 283,156 & 6.8E+7 & 0 & 6.1E+8\tabularnewline
\cline{2-8} 
 & \noun{SViews} & 9.61E+6 & 13.8 & 284,301 & 8.3E+6 & 5.9E+7 & 6.6E+8\tabularnewline
\cline{2-8} 
 & \noun{PViews} & 9.61E+6 & 13.4 & 284,301 & 8.3E+6 & 0 & 6.6E+8\tabularnewline
\hline 
\end{tabular}\hfill{}

\caption{\label{tab:monitor-2}Number of propagator executions (p), seconds
(t), fails (f), domain updates (u), calls to view object methods (calls),
and arithmetic operations (op) when solving three instances of the
Golomb problem. }
\end{table}
\selectlanguage{british}%

\section{\label{sec:Conclusion-and-Future}Conclusion and Future Work}

In this paper we addressed the modeling of decomposable constraints
and challenged the traditional view that such decomposition is best
made through the use of auxiliary variables. We showed that the propagation
of such decomposable constraints may be performed by view based propagators
that do not require such auxiliary variables, and discussed the properties
of several approximations of this approach. In particular, we focused
on models using box view-based propagators for pragmatic reasons,
and showed that, notwithstanding the fact that asymptotic worst-case
analysis leads to the same complexity, when applied to a comprehensive
set of benchmarks they perform significantly better than those based
on auxiliary variables, even when the latter models use stronger propagators. 

We have intentionally restricted the instantiation of view models
to those consisting only of box approximations, largely because box
view objects may be implemented efficiently. In fact, other view models
may be propagated using a similar approach, but creating efficient
corresponding view objects is much more challenging in general. As
an example, consider designing a domain view object, that is an object
which propagates constraints of the form $c\circ f_{i}\dots\circ f_{m}$
and computes a $\delta\delta$-approximation for every image and object
function of each of the functions $f_{i}$. Such object must thus
be able to compute $\delta$-domains, which cannot be done in constant
time for most functions $f_{i}$ as it is the case for $\beta$-domains.
But it will be interesting to check whether this could be done for
specific classes of functions, similarly to what was done in \cite{Tack2009}. 

Finally, given the superior performance of \noun{PViews} compared
with \noun{SViews}, one may question what to do when the former cannot
be used directly, as occurs when the problem at hand is specified
in an interpreted environment such as MiniZinc \cite{Nethercote2007}
where there is no C++ compilation involved. In this case, an option
is to generate C++ code for the constraints and compile it with a
JIT (Just In Time) C++ compiler. Pre-processing constraint problems
for generating specific solver binaries is shown in \cite{Balasubramaniam:2012:AAG:2337223.2337301}.
It would be interesting to see if the compilation time would compensate
the solving time in our case.

\bibliographystyle{abbrv}
\addcontentsline{toc}{chapter}{\bibname}\bibliography{bibliography}
\selectlanguage{british}%

\end{document}